\documentclass{article}

\PassOptionsToPackage{numbers,compress}{natbib}
\usepackage{xspace}
\usepackage{todonotes}

 \usepackage[preprint]{neurips_2026}

\usepackage[utf8]{inputenc} %
\usepackage[T1]{fontenc}    %
\usepackage{hyperref}       %
\usepackage{url}            %
\usepackage{booktabs}       %
\usepackage{amsfonts}       %
\usepackage{nicefrac}       %
\usepackage{microtype}      %
\usepackage{xcolor}         %

\usepackage{amsthm}
\usepackage{multirow}
\usepackage{wrapfig}
\usepackage{amssymb}
\usepackage{amsmath}
\usepackage{mathtools}
\usepackage{cleveref}
\makeatletter
\AddToHook{cmd/appendix/before}{\def\Cref@section@alias{appendix}\def\Cref@subsection@alias{appendix}}
\makeatother
\newcommand{\GaussianIG}{\textsc{GaussianIG}\xspace}
\usepackage{algorithm}
\usepackage{algpseudocode}

\newcommand{\enc}{\operatorname{enc}}
\newcommand{\entropy}{H}
\newcommand{\IG}{\operatorname{\widehat{IG}}}

\newtheorem{theorem}{Theorem}
\newtheorem*{theorem*}{Theorem}
\newtheorem{lemma}{Lemma}
\newtheorem*{lemma*}{Lemma}

\theoremstyle{definition}

\usepackage{bm} %
\usepackage{subcaption}

\usepackage{graphicx}

\algnewcommand{\Input}{\item[\textbf{Input:}]}
\algnewcommand{\Output}{\item[\textbf{Output:}]}
\usepackage{xcolor}
\newcommand{\bz}{{\bf z}}
\newcommand{\bZ}{{\bf Z}}

\usepackage[normalem]{ulem} 
\newif\ifshowchanges

\showchangestrue     %

\ifshowchanges

  \newcommand{\del}[1]{\textcolor{red}{\sout{#1}}}

\else

  \newcommand{\del}[1]{}

\fi

\algrenewcommand\algorithmiccomment[1]{\hfill{\color{black!55}\textit{$\triangleright$ #1}}}

\usetikzlibrary{arrows.meta}

\title{Measuring Semantic Progress in Multi-turn Dialogue via Information Gain}

\author{%
  Paul He\thanks{Work done while at Amazon Research, Tübingen, Germany} \\
  NTU Singapore\\
  \texttt{paul005@ntu.edu.sg} \\
  \\\And
  Shiva Kasiviswanathan\\
  Amazon\\
  \texttt{kasivisw@gmail.com} \\\And
  Dominik Janzing\\
  Amazon\\
  \texttt{janzind@amazon.com}
}

\begin{document}

\maketitle

\begin{abstract}
Evaluating multi-turn dialogue is challenging because quality emerges across turns rather than within individual responses. We focus on a key dimension of information-seeking dialogue: {\em semantic progress}, defined as the accumulation of new, question-relevant, and non-redundant information over the course of a conversation. We formalize semantic progress as question-conditioned uncertainty reduction and introduce an information-theoretic metric that approximates it in embedding space. Our main estimator uses a tractable Gaussian formulation with closed-form updates, while a complementary maximum-entropy argument shows why log-determinant structure arises more broadly when only second-order embedding information is retained. This formulation yields desirable theoretical properties, including monotonicity, additive decomposition of total information gain across turns, and diminishing returns for redundant evidence. Unlike LLM-as-a-judge approaches, our metric requires no autoregressive inference at evaluation time and is fully reproducible for a fixed embedding model. Experiments on MT-Bench, Chatbot Arena, and UltraFeedback show that the proposed metric achieves competitive agreement with human judgments despite targeting only semantic progress, with improved alignment on MT-Bench and UltraFeedback compared to several LLM-based judges. Notably, the method remains effective with lightweight embedding models under CPU-only execution, indicating that semantic progress can be captured without reliance on large model capacity.
\end{abstract}

\section{Introduction}
Large Language Models (LLMs) are increasingly deployed as conversational agents across diverse settings. Some interactions are open-ended, creative, or social, while others are primarily information-seeking, such as question answering and decision support \cite{li-etal-2023-autoconv, ye-etal-2023-enhancing, ma-etal-2025-large-language-models-meet, kamalloo-etal-2023-evaluating}. These use cases impose different requirements on evaluation, and no single metric can adequately capture all aspects of dialogue quality, including fluency, factuality, safety, style, user satisfaction, and task success.

In this work, we focus on a practically important subset: information-seeking dialogue. In such interactions, users iteratively refine their queries across turns, while agents provide clarifications, evidence, and partial answers that accumulate over the course of the conversation \cite{wu-etal-2023-inscit}. As a result, evaluating these dialogues requires more than assessing individual responses in isolation—it requires measuring how the conversation evolves over time \cite{deshpande-etal-2025-multichallenge, kim-etal-2022-mismatch}.

This temporal perspective is particularly important in real-world deployment. Developers often need to compare model variants, perform regression testing, monitor live systems, and analyze large volumes of user interactions for quality degradation. These workflows demand evaluation signals that are fast, reproducible, scalable, and informative at the level of entire dialogues \cite{deshpande-etal-2025-multichallenge, li-etal-2025-generation, guan2025evaluatingllmbasedagentsmultiturn}.

We focus on one such signal: whether an information-seeking dialogue accumulates new, question-relevant, and non-redundant information over time. We refer to this dimension as \emph{semantic progress}. Effective information-seeking dialogues should exhibit consistent progress across turns while avoiding redundancy and irrelevance \cite{kim-etal-2022-mismatch, finch-etal-2023-dont}. Unlike holistic notions of dialogue quality, semantic progress does not aim to measure fluency, style, safety, or factual correctness. Instead, it captures whether the conversation reduces uncertainty about the information needed to address the user’s request (see~\Cref{fig:intro-semantic-progress}).

We formalize semantic progress as \emph{question-conditioned semantic uncertainty reduction}: as a dialogue unfolds, accumulated answer evidence should reduce uncertainty along directions relevant to the user’s information need. In principle, this reduction could be defined over explicit facts, possible worlds, or structured answer states. However, such representations are typically unavailable in open-ended natural language settings. We therefore approximate this notion directly in embedding space, treating answer evidence as a sequence of observations that progressively shrink the covariance of a latent semantic state along question-relevant directions. The resulting information-gain score is deterministic, reference-free, and can be computed without autoregressive judge inference.

\begin{figure}[t]
\centering
\small
\begin{tikzpicture}[
    font=\small,
    >=latex,
    turnbox/.style={
        draw=black!45,
        line width=0.7pt,
        rounded corners=5pt,
        align=center,
        minimum width=2.25cm,
        minimum height=0.72cm,
        inner xsep=6pt,
        inner ysep=4pt,
        fill=#1,
        outer sep=0pt
    },
    scorebox/.style={
        draw=black!45,
        line width=0.7pt,
        rounded corners=6pt,
        align=center,
        minimum width=3.35cm,
        minimum height=1.25cm,
        inner xsep=7pt,
        inner ysep=6pt,
        fill=#1,
        outer sep=0pt
    },
    arrow/.style={
        -{Latex[length=2.1mm,width=1.6mm]},
        line width=0.85pt,
        draw=black!55,
        shorten >=2pt,
        shorten <=2pt
    },
    faintarrow/.style={
        -{Latex[length=2.1mm,width=1.6mm]},
        line width=0.85pt,
        draw=black!30,
        shorten >=2pt,
        shorten <=2pt
    }
]

\definecolor{cnovel}{HTML}{EEF6FF}
\definecolor{credundant}{HTML}{F7F7F7}
\definecolor{cscore}{HTML}{FFF8E8}
\definecolor{cellipseA}{HTML}{8ECAE6}
\definecolor{cellipseB}{HTML}{5DADE2}
\definecolor{cellipseC}{HTML}{2B6CB0}

\node[turnbox=cnovel] (turn1) at (0,1.35)
{\textbf{Turn 1}\\[-1pt]\scriptsize novel evidence};

\node[turnbox=credundant, text=black!55, draw=black!25] (turn2) at (0,0.45)
{\textbf{Turn 2}\\[-1pt]\scriptsize redundant evidence};

\node[text=black!45] at (0,-0.08) {$\vdots$};

\node[turnbox=cnovel] (turnT) at (0,-0.85)
{\textbf{Turn $T$}\\[-1pt]\scriptsize novel evidence};

\coordinate (center) at (3.35,0.35);

\draw[
    line width=1.6pt,
    draw=cellipseA,
    fill=cellipseA,
    fill opacity=0.12,
    rotate around={-18:(3.35,0.35)}
]
    (center) ellipse (1.25cm and 0.52cm);

\draw[
    line width=1.6pt,
    draw=cellipseB,
    fill=cellipseB,
    fill opacity=0.13,
    rotate around={-18:(3.35,0.35)}
]
    (center) ellipse (0.82cm and 0.36cm);

\draw[
    line width=1.8pt,
    draw=cellipseC,
    fill=cellipseC,
    fill opacity=0.16,
    rotate around={-18:(3.35,0.35)}
]
    (center) ellipse (0.48cm and 0.22cm);

\node[
    align=center,
    font=\scriptsize,
    text=black!60
] at (3.35,-0.82)
{semantic uncertainty shrinks\\[-1pt]
as useful evidence accumulates};

\draw[arrow] (turn1.east) -- (2.05,0.95);
\draw[faintarrow] (turn2.east) -- (2.05,0.45);
\draw[arrow] (turnT.east) -- (2.05,-0.05);

\draw[
    -{Latex[length=1.8mm,width=1.4mm]},
    line width=1.0pt,
    draw=cellipseC
] (3.35,0.35) -- (4.05,0.78);

\draw[
    -{Latex[length=1.8mm,width=1.4mm]},
    line width=0.9pt,
    draw=black!35
] (3.35,0.35) -- (4.05,0.38);

\node[scorebox=cscore] (score) at (6.95,0.35)
{\textbf{Semantic progress}\\[-1pt]
\scriptsize information gain\\[3pt]
\scriptsize accumulates novel evidence\\[-1pt]
\scriptsize with diminishing returns};

\draw[arrow] (4.85,0.35) -- (score.west);

\end{tikzpicture}

\caption{
Conceptual view of \emph{semantic progress} as captured by our information-gain metric. Useful turns introduce question-relevant evidence that reduces semantic uncertainty, while redundant turns yield diminishing marginal gains. Our metric approximates this uncertainty reduction in embedding space and accumulates it as dialogue-level information gain.
}
\label{fig:intro-semantic-progress}
\end{figure}
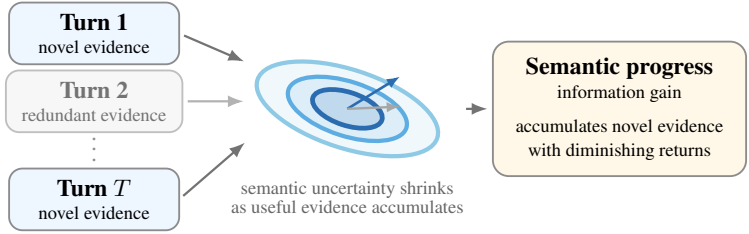

Our approach is complementary to, rather than a replacement for, holistic evaluators such as LLM-as-a-judge methods. While LLM judges capture broad aspects of dialogue quality, they are computationally expensive and sensitive to prompts, model versions, and stochastic inference. In contrast, our metric targets a specific dimension—semantic progress—while providing a fast, deterministic, and reproducible signal suitable for large-scale evaluation. We focus on semantic progress as one evaluation dimension; factuality, safety, coherence, and style are each studied by extensive evaluation literature \cite{hong-etal-2025-consistencychecker, li2024llmsasjudgescomprehensivesurveyllmbased} and are complementary to the signal proposed here. These external signals can easily be incorporated through additional weighting schemes, but evaluating those combinations is outside the scope of this work.
We make the following contributions:
\begin{itemize}
\item \textbf{Information-theoretic perspective.}
We formulate multi-turn dialogue evaluation as question-conditioned semantic uncertainty reduction, giving a model-agnostic definition of semantic progress as the accumulation of new, relevant, non-redundant information.

\item \textbf{Efficient, training-free estimator.}
We introduce a deterministic, reference-free Gaussian approximation in embedding space with closed-form updates and guarantees of monotonicity, telescoping aggregation, and diminishing returns.

\item \textbf{Empirical validation at scale.}
On MT-Bench, Chatbot Arena, and UltraFeedback, \GaussianIG reaches 84.03\%, 65.80\%, and
71.64\% agreement, respectively, while reducing evaluation time by 3--10x relative to hosted
LLM judges and supporting CPU-only execution.
\end{itemize}

\section{Related Work}
\noindent\textbf{Task-Oriented and Structured Dialogue Evaluation.}
In structured and task-oriented settings, dialogue quality is typically assessed using a combination of turn-level and dialogue-level signals, often augmented with learned evaluators or LLM-based judges. For instance, TD-EVAL employs a two-stage framework that combines turn-level metrics with dialogue-level aggregation and LLM judgments, achieving improved alignment with human preferences on benchmarks such as MultiWOZ and Tau-Bench \cite{budzianowski-etal-2018-multiwoz, yao2024taubenchbenchmarktoolagentuserinteraction, acikgoz-etal-2025-td}. These approaches are tailored to task-specific success criteria and generally depend on supervised models or external evaluators. In contrast, our method abstracts multi-turn dialogue progress as question-conditioned uncertainty reduction, providing a task-agnostic, training-free signal defined directly in embedding space.

\noindent\textbf{Dialogue-Level Coherence and Consistency Metrics.}
Another line of work evaluates dialogue quality using coherence, consistency, and contextual appropriateness across turns \cite{dey-etal-2022-towards, ghazarian-etal-2022-wrong}. These methods assess whether a dialogue remains internally consistent and contextually appropriate given its history. However, they do not explicitly model dialogue-level progress or account for diminishing returns from redundant information, which are central to our formulation.

\noindent\textbf{Large Language Models as Dialogue Evaluators.}
A substantial body of work explores large language models as dialogue evaluators, examining their agreement with human judgments and sensitivity to prompting and evaluation protocols \cite{chen-etal-2023-automatic}. LLM-as-a-judge approaches can achieve strong alignment with human preferences, as demonstrated by benchmarks such as MT-Bench 101, where carefully prompted models (e.g., GPT-4) attain high agreement on multi-turn dialogue quality \cite{Bai_2024}. However, these methods incur significant computational and operational costs \cite{jia-etal-2024-leveraging}, relying on large models, prompt engineering, and repeated autoregressive inference. This makes them expensive and often impractical for large-scale model comparison, regression testing, or continuous monitoring.

\noindent\textbf{Learned Evaluators and Judge Approximation.}
Closely related are learned evaluators that directly approximate human preferences or LLM-based judges, including dialogue-level predictors, pairwise comparison models, and feature-based frameworks \cite{ou-etal-2024-dialogbench, park2024paireval, zhou-etal-2024-llm-feature, li-etal-2025-exploring-reliability}. While often effective, these approaches require supervised training and inherit the opacity, bias, and distributional assumptions of the signals they are trained to replicate.

\noindent\textbf{Positioning of Our Work.}
Prior work spans task-specific evaluators for structured dialogue, coherence and consistency metrics, LLM-as-a-judge approaches, and learned evaluators that approximate human or model judgments. These methods either target holistic quality dimensions or rely on supervised training and prompt-dependent inference, often incurring substantial computational cost and limited reproducibility \cite{li-etal-2025-generation, 10.1609/aaai.v38i17.29923, li-etal-2025-exploring-reliability}. In contrast, we isolate a complementary dimension, \emph{semantic progress}, and formalize it as question-conditioned uncertainty reduction. Our approach is task-agnostic, training-free, and avoids autoregressive inference, producing a fast, deterministic, and reference-free dialogue-level signal. It further provides explicit theoretical guarantees, including monotonicity and diminishing returns, enabling scalable and reliable evaluation pipelines while remaining complementary to holistic LLM-based judges.

\section{Setting the Stage}
Let $\mathcal A$ be a finite \textbf{alphabet} and $\mathcal A^*$ the set of all finite strings over $\mathcal{A}$.
A $T$-turn \textbf{question--answer dialogue} is a sequence $D_{1:T} \triangleq \big((q_1,a_1),\ldots,(q_T,a_T)\big)$ where $(q_t,a_t)\in \mathcal A^*\times \mathcal A^*$ denotes question and answer in step $t \in [T]$. We assume each answer can be decomposed into a finite collection of smaller semantic units, which we refer to as \textbf{evidence items}. Let $\mathcal{E}$ denote the universe of all possible evidence items (e.g., semantic chunks of responses). Let $E_t \subseteq \mathcal{E}$ denote the \textbf{set of evidence} extracted from answer $a_t$. We define $A_t \triangleq \bigcup_{s=1}^t E_s$ as the \textbf{accumulated evidence} observed up to turn $t$. Let $q$ denote the user's information need (e.g., dialogue context, question). We define a \textbf{relevance function} $\rho_q:\mathcal{E} \to \mathbb{R}_{\geq 0}$, where $\rho_q$ measures how relevant an evidence item $e$ is to $q$. We consider dialogue progress as a function of accumulated evidence. A \textbf{semantic progress measure} is a function $\mathcal{P}_q: 2^\mathcal{E} \to \mathbb{R}_{\geq 0}$ with marginal gain
\begin{align}
    \Delta_q(e \mid A) = \mathcal{P}_q(A \cup \{e\}) - \mathcal{P}_q(A)
\end{align} A dialogue prefix $D_{1:t}$ is scored by first extracting its accumulated evidence $A_t$, then evaluating $\mathcal P_q(A_t)$ relative to the information need $q$. A useful measure of semantic progress in information-seeking dialogue should satisfy \textbf{irrelevance} (if $\rho_q(e)=0$, then $\Delta_q(e \mid A)=0$), \textbf{monotonicity} ($\Delta_q(e \mid A) \geq 0$), and \textbf{diminishing returns} (if $A \subseteq B$ then $\Delta_q(e \mid A) \geq \Delta_q(e \mid B) $). These capture that progress accumulates with relevant evidence, does not decrease when adding information, and discounts redundancy.

\textbf{User Perspective and Agent-Induced Information Gain.} Dialogue information gain can be viewed from two perspectives. A third-party observer may learn from both user questions and agent answers. From the user's perspective, however, the question is already known and should condition the current information need. Since we evaluate agent progress, we adopt the user perspective.

Let $C_{t-1}^A \triangleq (q_1,a_1,\ldots,q_{t-1},a_{t-1})$ denote the context before the user’s $t$-th turn. After the user asks $q_t$, define $C_t^U \triangleq (C_{t-1}^A,q_t),$ and after the agent answers $a_t$, define $C_t^A \triangleq (C_t^U,a_t)$. Thus $C_t^U $is the context available immediately before the agent answer, and $C_t^A$ is the context after the full turn. Let $W$ be a latent task-relevant state. By the chain rule for mutual information,
\begin{align}
    I(W;D_{1:T})
=
\sum_{t=1}^T I(W;q_t\mid C_{t-1}^A)
+
\sum_{t=1}^T I(W;a_t\mid C_t^U).
\end{align} 
The first term is user-induced task specification; the second is agent-induced semantic progress. The second term is not meant to equal the user's subjective belief update, since
the user may have private knowledge not manifested in the transcript. Rather, we
use it as a transcript-level proxy for agent-induced progress: the information
added by the agent's answers after conditioning on the expressed user need. We therefore use
$G^\star(D_{1:T}) \triangleq \sum_{t=1}^T I(W; a_t \mid C^U_t)$
as an idealized transcript-level proxy for agent-induced semantic progress.
This decomposition allows $q_t$ to depend on earlier answers, since $C_{t-1}^A$ contains the full dialogue history. It also avoids penalizing the agent when the user makes the task broader or more ambiguous. Our estimator approximates $G^\star$ by treating agent answers as evidence weighted by relevance to the current user turn.

\textbf{A Hypothesis.}
Let $q$ denote a user information need, and let $D^A_{1:T}$ and $D^B_{1:T'}$ be two information seeking dialogues for the same $q$. Let $A^A_T$ and $A^B_T$ denote their accumulated evidence. Let $W$ be a latent random variable representing the task-relevant state of the world (i.e., the information needed to resolve the user’s request). We write $H(\cdot)$ for Shannon entropy of discrete variables and
$h(\cdot)$ for differential entropy of continuous variables.
We define the realized ideal semantic progress of an observed evidence accumulation $a_T$ as
\begin{align}
    \mathcal P^\star_q(a_T) \triangleq H(W\mid q)-H(W\mid q,A_T=a_T)
\end{align}
where its expectation over the random evidence accumulation $A_T$ is the conditional mutual information $\mathbb{E}_{A_T}[\mathcal P^\star_q(A_T)] = I(W;A_T\mid q).$
If human preference between the two dialogues is primarily driven by the accumulation of new, question relevant, non-redundant information, then the preferred dialogue is most likely to satisfy 
$\mathcal P_q^\star(a_T^A) > \mathcal P_q^\star(a_T^B)$  where \(a^A_T\) and \(a^B_T\) are the observed accumulated evidence sets. Since $\mathcal P_q^\star$ is not directly computable in open-ended dialogue, we test this hypothesis using a tractable estimator $\widehat{\mathcal P}_q$, predicting the dialogue with higher $\widehat{\mathcal P}_q$ as preferred and measuring agreement with human judgments.

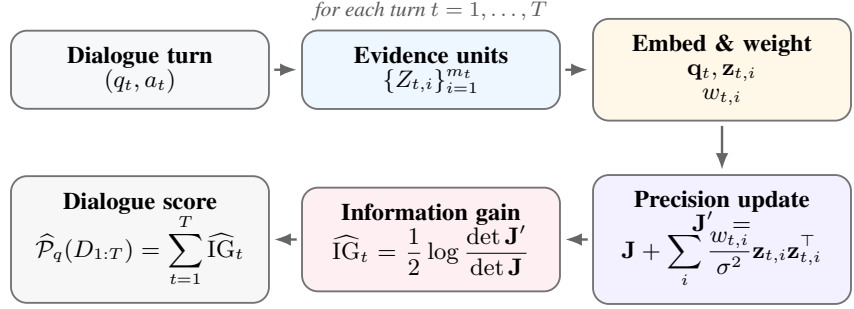
\begin{figure}[t]
\centering
\begin{tikzpicture}[
    font=\small,
    >=latex,
    box/.style={
        draw=black!55,
        line width=0.7pt,
        rounded corners=6pt,
        align=center,
        minimum height=0.95cm,
        inner xsep=8pt,
        inner ysep=6pt,
        text width=2.85cm,
        outer sep=0pt,
        fill=#1
    },
    arrow/.style={
        -{Latex[length=2.2mm,width=1.7mm]},
        line width=0.9pt,
        draw=black!60,
        shorten >=2pt,
        shorten <=2pt
    }
]

\definecolor{cinput}{HTML}{F6F8FA}
\definecolor{cevidence}{HTML}{EEF6FF}
\definecolor{cweight}{HTML}{FFF8E8}
\definecolor{cupdate}{HTML}{F3F0FF}
\definecolor{cscore}{HTML}{FDEFF2}
\definecolor{ctotal}{HTML}{F7F7F7}

\node[box=cinput] (input) at (0,0)
{\textbf{Dialogue turn}\\[-1pt]$(q_t,a_t)$};

\node[box=cevidence] (evidence) at (3.85,0)
{\textbf{Evidence units}\\[-1pt]$\{Z_{t,i}\}_{i=1}^{m_t}$};

\node[box=cweight] (weight) at (7.70,0)
{\textbf{Embed \& weight}\\[-1pt]
$\mathbf q_t,\mathbf z_{t,i}$\\[-1pt]
$w_{t,i}$};

\node[box=ctotal] (total) at (0,-2.25)
{\textbf{Dialogue score}\\[-1pt]
$\displaystyle \widehat{\mathcal P}_q(D_{1:T})
=\sum_{t=1}^T\widehat{\mathrm{IG}}_t$};

\node[box=cscore] (score) at (3.85,-2.25)
{\textbf{Information gain}\\[-1pt]
$\displaystyle \widehat{\mathrm{IG}}_t=
\frac12\log\frac{\det\mathbf J'}{\det\mathbf J}$};

\node[box=cupdate] (update) at (7.70,-2.25)
{\textbf{Precision update}\\[-1pt]
$\displaystyle \mathbf J'=\mathbf J+\sum_i
\frac{w_{t,i}}{\sigma^2}
\mathbf z_{t,i}\mathbf z_{t,i}^{\top}$};

\draw[arrow] (input.east) -- (evidence.west);
\draw[arrow] (evidence.east) -- (weight.west);
\draw[arrow] (weight.south) -- (update.north);
\draw[arrow] (update.west) -- (score.east);
\draw[arrow] (score.west) -- (total.east);

\node[
    font=\footnotesize\itshape,
    text=black!65
] at (3.85,0.78)
{for each turn $t=1,\ldots,T$};

\end{tikzpicture}

\caption{
Algorithmic overview of Gaussian information-gain evaluation. Each answer is split into evidence units, embedded, weighted by question relevance, and incorporated through a weighted precision update. The resulting log-determinant change gives turn-level information gain, which is summed into the dialogue-level semantic progress score.
}
\label{fig:overview}
\end{figure}

\section{Our Method}
As an idealized symbolic analogy, suppose there is a finite hypothesis space of
possible worlds, a fixed background theory, a uniform posterior over worlds
consistent with the accumulated facts $F_t$, and deterministic evidence that only
rules out inconsistent worlds. If $\mathcal M(F_t)$ denotes the surviving worlds
after turn $t$, then the turn-level gain is
$\IG_t = \log \frac{|\mathcal M(F_{t-1})|}{|\mathcal M(F_t)|}$. This captures uncertainty reduction, telescoping progress, and diminishing returns for repeated evidence. Since natural-language does not provide explicit facts, a background theory, or tractable model counts, we approximate this idea with embeddings. \Cref{fig:overview} summarizes the proposed pipeline, showing how dialogue turns are decomposed into evidence, embedded, and aggregated via precision updates. The full procedure is given in~\Cref{alg:ig} (Appendix~\ref{app:exp}).
The ideal score $\mathcal P_q^\star(a_T)$ measures information gained about a latent task-relevant state, but requires symbolic evidence, a posterior over possible worlds, and model counting. Our estimator replaces this symbolic uncertainty with Gaussian uncertainty in embedding space. The resulting score, denoted $\widehat{\mathcal P}_q(D_{1:T})$, is deterministic, reference-free, and computed from the dialogue text alone.

\subsection{Gaussian Approximation to Information Gain}
\label{sec:approx_ig}

For each turn $t$, we extract a finite collection of evidence strings $\mathrm{Z}_t \triangleq \{\mathrm{Z}_{t,1},\ldots,\mathrm{Z}_{t,m_t}\}\subseteq \mathcal A^* $ from the answer $a_t$. In our experiments, these evidence strings are sentence-level units (details later in the paper) 
though other deterministic segmentations can be used. This gives a concrete approximation to the abstract evidence set $E_t $. We write $\enc:\mathcal A^*\to \mathbb{R}^d$ for a fixed text embedding function, and denote
\begin{align*}
 \mathbf{q}_t \triangleq \enc(q_t)\in\mathbb{R}^d \quad \mathbf{z}_{t,i} \triangleq \enc(\mathrm{Z}_{t,i})\in\mathbb{R}^d .
\end{align*}
for the question and evidence embeddings. We introduce a latent semantic uncertainty variable
$\mathbf S\in\mathbb R^d$,
which represents unresolved semantic uncertainty in the embedding space. This variable is an abstract modeling device, not a literal world state or fact representation. We place an isotropic Gaussian prior
\begin{align*}
\mathbf{S}\sim \mathcal{N}(\mathbf{0},\mathbf{\Sigma}_0),
\qquad \mathbf{\Sigma}_0 \triangleq \sigma_0^2 \mathbf{I}.
\end{align*}
We justify the resulting log-determinant score through a maximum-entropy argument, formalized in \Cref{thm:gaussian-maxent-upper-bound}. Among continuous distributions with fixed covariance, the Gaussian maximizes differential entropy.
Since our estimator retains only second-order embedding statistics, Gaussian entropy gives a least-committal uncertainty surrogate, and entropy reduction becomes a log-determinant covariance ratio.

Each evidence embedding induces a scalar linear measurement of the \textbf{latent semantic state}:
\begin{align}
    y_{t, i} = \mathbf{z}_{t,i}^\top \mathbf{S}   + \bm{\varepsilon}_{t,i},
 \quad \varepsilon_{t,i} \sim \mathcal{N}(0, \sigma^2)
,
\end{align}
Appendix~\ref{app:queries1d} provides intuition for viewing an embedding direction as a one-dimensional measurement of $\mathbf S$. Since the posterior covariance in Bayesian linear regression depends on the measurement directions and noise variances, but not on the observed scalar values, the specific values $y_{t,i}$ do not need to be observed or estimated; see Appendix~\ref{appendix:blr}. To instantiate \textbf{question-conditioned relevance}, we assign each evidence unit a nonnegative weight $w_{t, i} = r^{\mathrm{rel}}_{t, i}r^{\mathrm {qual}}_{t, i}$ where
\begin{align}
    r^{\mathrm{rel}}_{t, i}
=
\max\!\left(0,\langle \widehat{\mathbf q}_t,\widehat{\mathbf z}_{t,i}\rangle\right)^\beta 
\end{align}
where $\widehat{\mathbf v}=\mathbf v/\|\mathbf v\|_2$ and $r^{\mathrm{qual}}_{t, i} \in [0, 1]$ is an optional external evidence-quality factor. In this work, set $r^{\mathrm{qual}}_{t, i} = 1$ and evaluate relevance-weighted semantic progress only. Optionally, we set $w_{t,i}=0$ when $w_{t,i}<\eta$. The weight $w_{t,i}$ is the embedding-space instantiation of the relevance function $\rho_q(e)$: evidence weakly aligned with the question contributes little or no posterior update, while aligned evidence reduces uncertainty along its semantic direction. 

Let $\mathbf{J}_t \triangleq \mathbf{\Sigma}_t^{-1}$ denote the precision matrix of the posterior
distribution over the latent semantic uncertainty variable $\mathbf{S}$ after incorporating
evidence up to turn $t$.
For nonnegative evidence weights $w_{t,i}\ge 0$, the posterior precision updates additively:

\begin{align}
\label{eq:precision_update}
\mathbf{J}_t = \mathbf{J}_{t-1} + \sum_{i=1}^{m_t} \frac{w_{t,i}}{\sigma^2}\,\mathbf{z}_{t,i}\mathbf{z}_{t,i}^\top .
\end{align}
This precision update is the information-form posterior covariance update for Bayesian
linear regression with evidence embeddings as measurement directions. The relevance weights
can be equivalently interpreted as heteroscedastic observation noise, with
$\varepsilon_{t,i}\sim\mathcal N(0,\sigma^2/w_{t,i})$ for $w_{t,i}>0$, and $w_{t,i}=0$
contributing no update. Appendix~\ref{appendix:blr} gives the derivation. From now on, denote $\alpha_{t, i} = \frac{w_{t,i}}{\sigma^2}$. The differential entropy of a $d$-dimensional Gaussian $\mathcal{N}(\bm{\mu},\mathbf{\Sigma})$ is
\begin{align}
h(\mathbf{S}) = \frac{d}{2}\log(2\pi e)+\frac{1}{2}\log\det(\mathbf{\Sigma}).
\end{align}
Since the first term is constant in $t$, the per-turn information gain induced by turn $t$ is
\begin{align}
\label{eq:gauss_ig}
\IG_t 
\triangleq h(\mathbf{S}\mid \mathrm{Z}_{1:t-1}) - h(\mathbf{S}\mid \mathrm{Z}_{1:t})
= \frac{1}{2}\log \frac{\det(\mathbf{\Sigma}_{t-1})}{\det(\mathbf{\Sigma}_t)} .
\end{align}
which gives us the dialogue level estimator: $\widehat{\mathcal P}_q(D_{1:T})
\triangleq
\sum_{t=1}^T \widehat{\mathrm{IG}}_t.$

The log-determinant score can be motivated independently of the Gaussian posterior model. In open-ended dialogue, we do not observe symbolic facts or a tractable posterior over possible worlds; instead, our estimator $\IG$ retains only second-order structure of the evidence embeddings. Under such a covariance-only description, the Gaussian distribution is the maximum-entropy distribution, so $\log\det\bm \Sigma$ is the least-committal surrogate for uncertainty volume. \Cref{thm:gaussian-maxent-upper-bound} formalizes this. It shows that, when evidence is represented only through its empirical covariance, the Gaussian log-determinant quantity upper-bounds the information that
the evidence can reveal about the latent semantic state, up to a conditional-uncertainty term.
\begin{theorem}
    \label{thm:gaussian-maxent-upper-bound}
    For a fixed turn $t \in [T]$, let $\mathbf{Z}$ be the random evidence vector corresponding to $\mathbf{z} = \langle \mathbf{z}_1, \ldots,  \mathbf{z}_m \rangle \in \mathbb{R}^{md}$ with evidence components $\mathbf{Z}_i \in \mathbb{R}^d$. If $h(\mathbf{Z} \mid \mathbf{S})\geq c$, where $c$ is independent of the dialogue strategy, and the evidence has common empirical covariance $\mathbf{\Sigma}_\mathbf{Z} \succ \mathbf{0}$, then:
    \begin{align}
        I(\mathbf{Z};\mathbf{S}) \leq m \big [\frac{d}{2}\log(2\pi e) + \frac{1}{2}\log\det \mathbf{\Sigma}_\mathbf{Z} \big ] -c 
    \end{align}
\end{theorem}
We prove this in \Cref{app:proof_max_ent}. This
reflects the information-theoretic principle that high-information evidence must be sufficiently non-compressible. In our setting, larger effective covariance volume corresponds to evidence that can potentially carry more information about the latent semantic state.

This estimator $\IG$ follows the same structure as the ideal symbolic formulation: evidence accumulates over turns, uncertainty decreases through posterior updates, and dialogue-level progress is the total reduction in uncertainty. The next results show that the estimator preserves the desired structural properties of semantic progress.

\subsection{Structural Properties of the Estimator}
The Gaussian estimator inherits the structural properties desired of a semantic progress measure. Because each evidence item contributes a nonnegative rank-one update to the precision matrix, posterior uncertainty can only decrease. This gives monotonicity and a telescoping dialogue-level score. Proofs of the theorems are provided in Appendix~\ref{appendix:proof_monotonicity}, and Appendix~\ref{appendix:proof_submodularity}.

\begin{theorem}[Monotonicity and Telescoping]
\label{thm:monotonicity}
Assume $\mathbf{\Sigma}_0 \succ \mathbf{0}$ and weights $\alpha_{t,i}\ge 0$.
Under the precision update in \Cref{eq:precision_update}, we have
$\mathbf{\Sigma}_t \preceq \mathbf{\Sigma}_{t-1}$ (in PSD order) and hence $\IG_t\ge 0$ for all $t$.
Moreover, the total gain telescopes:
\begin{align}
\sum_{t=1}^T \IG_t
=
\frac{1}{2}\log \frac{\det(\mathbf{\Sigma}_0)}{\det(\mathbf{\Sigma}_T)} ,
\end{align}
so it depends only on the initial and final posterior covariances.
\end{theorem}
Thus, adding evidence cannot reduce the score, and the dialogue-level value is exactly the total reduction in posterior uncertainty. We next show that the estimator also discounts redundancy. Let $\mathcal X$ be a finite set of evidence items, where each item $i\in\mathcal X$ has embedding $\mathbf z_i\in\mathbb R^d$ and weight $\alpha_i\ge 0$. For any subset $\mathcal S\subseteq\mathcal X$, define
$\mathbf J(\mathcal S)
=
\mathbf J_0
+
\sum_{i\in\mathcal S}
\alpha_i\mathbf z_i\mathbf z_i^\top$ and 
$F(\mathcal S)
=
\log\det \mathbf J(\mathcal S)$.

\begin{theorem}[Diminishing Returns]
\label{thm:submodularity}
Assume $\mathbf J_0\succ 0$. The set function $F$ is monotone submodular. That is, for all
$A\subseteq B\subseteq\mathcal X $and $x\notin B$,
\begin{align*}
    F(A\cup\{x\})-F(A)
\ge
F(B\cup\{x\})-F(B).
\end{align*}
\end{theorem}
Since total Gaussian information gain is proportional to $\log\det \mathbf J(\mathcal S)$ up to constants, \Cref{thm:submodularity} implies diminishing marginal returns: evidence contributes less when similar evidence has already been accumulated. This is the formal mechanism by which the estimator avoids rewarding redundancy or verbosity alone. 

\textbf{Implication.} The estimator exhibits the desired properties for semantic-progress evaluation: information gains compose across turns, irrelevant evidence contributes negligibly, and redundant evidence yields diminishing marginal returns. Consequently, high scores arise only from sustained, question-relevant reduction in semantic uncertainty, rather than verbosity or repetition.

\subsection{Local Behavior: Irrelevance and Repetition}
The previous subsection establishes global properties. We now make the local behavior of a single evidence item explicit. For an evidence embedding $\mathbf z$ with weight $\alpha\ge 0$, define its marginal contribution at precision $\mathbf J$ as
\begin{align}
\label{eq:delta_ig}
\Delta(\mathbf z;\mathbf J)
&\triangleq
\log\det(\mathbf J+\alpha\mathbf z\mathbf z^\top)
-
\log\det(\mathbf J) \\
&=
\log\!\left(1+\alpha\mathbf z^\top\mathbf J^{-1}\mathbf z\right),
\end{align}
where we used the matrix determinant lemma in the last equality. Thus, if the evidence is completely irrelevant under the question-conditioned weighting, then $\alpha=0$ and its gain is exactly zero. If $\alpha$ is small, the gain is correspondingly bounded. Proofs of the lemmas are provided in Appendix~\ref{appendix:soft_irrelev} and Appendix~\ref{appendix:proof_redundancy}.
\begin{lemma}[Soft Irrelevance]
\label{lemma:soft_irrelev}
Assume embeddings are norm bounded $\|\mathbf{z}\| \leq B$, and $\alpha\leq \varepsilon$.
When $\mathbf{z}$ is irrelevant compared to the question $\mathbf{q}$, then information gain for that turn is bounded by $\log(1+\alpha\lambda_{\max}(\mathbf{J}_0^{-1})B^2) = \mathcal{O}(\varepsilon)$.
\end{lemma}

\begin{lemma}[Redundancy leads to diminishing returns]
\label{lemma:redundancy}
Consider repeatedly adding the same evidence vector $z$ for $k$ times with the same $\alpha$. Then, $\Delta(\mathbf{z};\mathbf{J}_{k}) \geq \Delta(\mathbf{z};\mathbf{J}_{k+1}) $
\end{lemma}
Together, these lemmas show that weakly relevant evidence contributes little, while repeated evidence has decreasing marginal value. In Appendix~\ref{appendix:toy_examples}, we verify these behaviors in controlled synthetic dialogues, including cases where novel information is reintroduced after redundant or irrelevant turns.

\textbf{Implication.} These local properties are well-suited for semantic-progress evaluation: irrelevant content yields negligible gain, while repeated or paraphrased content is discounted even when on-topic. Consequently, high scores require evidence that is both question-relevant and genuinely informative relative to the dialogue history. Combined with the global properties, the estimator rewards sustained semantic progress, favoring evidence that continues to reduce uncertainty beyond what the dialogue has already established.

\section{Experiments}
\subsection{Main Results}
We evaluate Gaussian information gain as a signal for semantic progress in information-seeking dialogue. Our experiments first ask whether the score aligns with preference judgments and whether it offers operational speed advantages over LLM-as-a-judge evaluation. We then study robustness across embedding backends and controlled cases where semantic relevance and factual correctness diverge in \Cref{sec:ablation}. Additional experimental details are provided in Appendices~\ref{app:exp} and~\ref{app:syn}.

\begin{table*}[t]
\centering
\setlength{\tabcolsep}{4pt}
\caption{
Agreement and runtime comparison on MT-Bench, Chatbot Arena, and UltraFeedback.
Agr. denotes agreement with the preference label, $\tau$ denotes Kendall's rank correlation, and time denotes end-to-end wall-clock seconds for $N=100$ dialogue pairs averaged over $R=5$ runs. LLM-as-a-judge baselines were executed via AWS Bedrock using temperature $0$.
}
\label{tab:main_combined}
\begin{tabular}{lccc|ccc|ccc}
\toprule
\multirow{2}{*}{\textbf{Method}}
& \multicolumn{3}{c|}{\textbf{MT-Bench}}
& \multicolumn{3}{c|}{\textbf{Chatbot Arena}}
& \multicolumn{3}{c}{\textbf{UltraFeedback}} \\
\cmidrule(lr){2-4} \cmidrule(lr){5-7} \cmidrule(lr){8-10}
& Agr. $\uparrow$ & $\tau$ $\uparrow$ & Time $\downarrow$
& Agr. $\uparrow$ & $\tau$ $\uparrow$ & Time $\downarrow$
& Agr. $\uparrow$ & $\tau$ $\uparrow$ & Time $\downarrow$ \\
\midrule

Mistral Large 3 \cite{mistral3}
& 76.50 & 0.54 & 132
& 60.19 & 0.20 & 119
& 70.75 & 0.42 & 153 \\

DeepSeek R1 \cite{deepseekai2025deepseekr1incentivizingreasoningcapability}
& 80.96 & 0.61 & 633
& 66.14 & 0.32 & 480
& 71.43 & 0.43 & 1173 \\

GPT OSS 120b \cite{openai2025gptoss}
& 72.27 & 0.44 & 568
& 65.27 & 0.31 & 328
& 70.75 & 0.42 & 408 \\

Claude Sonnet 3.7 \cite{anthropic2025claude37}
& 76.47 & 0.53 & 260
& 64.74 & 0.29 & 251
& 70.52 & 0.40 & 285 \\

Claude Sonnet 4 \cite{anthropic2025claude4}
& 81.93 & 0.64 & 220
& 65.58 & 0.31 & 175
& 69.84 & 0.38 & 535 \\

Claude Sonnet 4.5 \cite{anthropic2025claude45}
& 76.05 & 0.52 & 342
& \textbf{66.42} & \textbf{0.33} & 286
& 69.16 & 0.38 & 623 \\

\midrule
\textbf{Ours (\GaussianIG)}
& \textbf{84.03} & \textbf{0.68} & \textbf{86}
& 65.80 & 0.32 & \textbf{44}
& \textbf{71.64} & \textbf{0.43} & \textbf{104} \\
\bottomrule
\end{tabular}

\end{table*}

\noindent\textbf{Setup.} We evaluate on MT-Bench, Chatbot Arena \cite{zheng2023judging}, and UltraFeedback \cite{10.5555/3692070.3692454}. MT-Bench and Chatbot Arena provide paired dialogue comparisons with human preference votes. UltraFeedback provides scored responses across several quality dimensions. Following prior work, we evaluate only examples with non-tied preference labels. We do not filter by category, which makes the evaluation conservative with respect to our stated scope: \GaussianIG measures semantic progress rather than holistic quality, and is expected to be most meaningful for information-seeking interactions.

For each dialogue $D_{1:T}$, we compute the dialogue-level score $\widehat{\mathcal P}_q(D_{1:T})
=
\sum_{t=1}^T \widehat{\mathrm{IG}}_t$. 
Given a pair $(D^A_{1:T},D^B_{1:T})$, we predict the preferred dialogue by computing $\arg\max_{X \in \{A,B\}} \widehat{\mathcal P}_q(D^X_{1:T})$.
We report agreement with the majority preference label and Kendall's $\tau$. Unless otherwise stated, we use \texttt{Qwen3-Embedding-0.6B}; \Cref{sec:ablation} studies other embedding backends and model sizes.

We compare against LLM-as-a-judge baselines, which recent surveys identify as among the strongest automatic dialogue evaluators \cite{li-etal-2025-generation, 10.1609/aaai.v38i17.29923}. These baselines are holistic evaluators, while our metric targets a narrower quantity: agent-induced semantic progress through the accumulation of relevant, non-redundant information. Runtime is measured on a fixed $N=100$ dialogue-pair subset sampled from the same set with fixed ordering and averaged over $R=5$ runs. Our method performs local embedding passes and closed-form Gaussian updates, while judge baselines call hosted inference APIs, so runtimes reflect operational deployment overhead such as network latency and service execution time.

\noindent\textbf{Results.} \Cref{tab:main_combined} reports agreement and runtime on three preference benchmarks.  Since the LLM-as-a-judge baselines evaluate broad dialogue quality, while \GaussianIG measures only agent-induced semantic progress, the comparison is not intended as a holistic evaluator leaderboard. Instead, it tests whether semantic progress explains a substantial fraction of human preference judgments, including in mixed settings where creativity, tone, concision, safety, or style may also drive preferences. On MT-Bench, \GaussianIG obtains $84.03\% $ agreement, comparable to or higher than the evaluated LLM-judge baselines; on UltraFeedback, it also achieves strong agreement, while on Chatbot Arena it remains competitive. These results support \GaussianIG as a complementary metric for relevance-weighted, non-redundant information gain. The dimension-level UltraFeedback analysis further shows weaker alignment with truthfulness and honesty than with helpfulness, reinforcing that factuality and honesty remain outside the scope of the metric.

The runtime results show the main operational advantage. \GaussianIG is substantially faster than hosted LLM judges across benchmarks: $86$s versus $132$--$633$s on MT-Bench, $44$s versus $119$--$480$s on Chatbot Arena, and $104$s versus $153$--$1173$s on UltraFeedback for $N=100$ dialogue pairs. This efficiency comes from replacing autoregressive judge inference with embedding computation and closed-form rank-one covariance updates.

\subsection{Ablations and Robustness}
\label{sec:ablation}

\textbf{Embedding Model.}
\Cref{tab:embedding_ablation} studies sensitivity to the embedding backend. The evaluated models span over three orders of magnitude in parameter count, from a $\sim$2M-parameter CPU model to a 4B-parameter embedding model.

Performance is stable across a wide range of embedding models. On MT-Bench, even the $\sim$2M-parameter CPU model achieves $80.25\%$ agreement, while several small CPU models approach or exceed the default \texttt{Qwen3-Embedding-0.6B} configuration. Larger models provide only modest improvements at substantially higher computational cost. This suggests that the uncertainty-reduction signal saturates once the embedding space captures the dominant semantic directions, after which additional model capacity yields limited benefit for the Gaussian update. Appendix~\ref{app:large_models} discusses this saturation effect in more detail.

\begin{table*}[t]
\centering
\caption{
Embedding model ablation across benchmarks of our method \GaussianIG. Here, Agr.\ denotes pairwise agreement. Device indicates whether embeddings were computed on CPU or GPU/MPS. Runtimes are end-to-end wall-clock seconds for $N=100$ dialogue pairs, averaged over $R=5$ runs.
}
\label{tab:embedding_ablation}
\begin{tabular}{lc|cc|cc|cc}
\toprule
\textbf{Model} & \textbf{Device}
& \multicolumn{2}{c|}{\textbf{MT-Bench}}
& \multicolumn{2}{c|}{\textbf{Chatbot Arena}}
& \multicolumn{2}{c}{\textbf{UltraFeedback}} \\

&
& Agr. $\uparrow$ & Time $\downarrow$
& Agr. $\uparrow$ & Time $\downarrow$
& Agr. $\uparrow$ & Time $\downarrow$ \\

\midrule

Potion-2M \cite{minishlab2024model2vec} & CPU
& 80.25 & 7.1
& 66.05 & 0.4
& 71.20 & 0.8 \\

MiniLM-23M \cite{wang2020minilmdeepselfattentiondistillation} & CPU
& 83.61 & 17.9
& 66.27 & 6.77
& 68.93 & 21.5 \\

Arctic-32M \cite{merrick2024arcticembedscalableefficientaccurate} & CPU
& 85.29 & 29.3
& 65.73 & 13.4
& 69.84 & 37.1 \\

\midrule

ModernBERT-0.1B \cite{nussbaum2024nomic} & GPU
& 83.19 & 28.5
& 65.75 & 15.5
& 68.03 & 57.0 \\

MicroLlama-0.3B \cite{wang2024microllama} & GPU
& 84.45 & 47.4
& 66.17 & 23.7
& 67.80 & 73.2 \\

Qwen3-0.6B \cite{qwen3embedding} & GPU
& 84.03 & 86.4
& 65.80 & 44.1
& 71.64 & 104.4 \\

Qwen3-4B \cite{qwen3embedding} & GPU
& 84.83 & 396.2
& 66.13 & 246.3
& 70.07 & 642.0 \\

\bottomrule
\end{tabular}

\end{table*}

\textbf{Cosine-similarity Only.}
To isolate the contribution of the information-gain mechanism, we compare against a cosine-only ablation that sums per-turn cosine similarities between responses and the dialogue context, removing uncertainty accumulation and diminishing returns. On clear-majority MT-Bench pairs, this baseline achieves only $44.5\%$ agreement, below both a majority baseline $(51.7\%)$ and random guessing $(50.1\%\pm3.2)$, and far below \GaussianIG $(84.0\%)$. This indicates that the gains are not explained by embedding similarity alone: the history-dependent covariance update and redundancy discounting are essential.

\textbf{Length and Padding Robustness.}
To test whether $\widehat{\mathcal P}$ rewards length rather than semantic progress, we pad UltraFeedback responses with irrelevant, redundant, or novel relevant content. At $r=2.0$, irrelevant and redundant padding have small effects ($0.032\pm0.010$ and $0.111\pm0.016$), while novel relevant padding substantially increases the score ($5.697\pm0.408$). This suggests $\widehat{\mathcal P}_q$ is sensitive to question-relevant evidence rather than length alone. See Appendix~\ref{app:padding} for details.

\textbf{Hallucinations and Relevance Parameters.}
We evaluate whether the relevance weighting is sensitive to factual perturbations that preserve surface form and semantic uncertainty. We vary the relevance cutoff $\eta$ and exponent $\beta$.  With mild relevance weighting $\beta=1,\eta=0.05$ the metric already prefers the non-corrupted response above chance. Increasing $\beta$ generally improves sensitivity to the factual perturbations when the cutoff $\eta$ is not too aggressive. Several settings with $\beta \in \{3, 4, 5, 6\}$ achieve win rates between $0.85$ and $0.93$. However, performance drops sharply for large $\eta$, especially at high $\beta$, because the filter removes too much evidence. We provide more details in Appendix~\ref{app:synth}.

\section{Conclusion}
This work develops an information-theoretic approach to multi-turn dialogue evaluation, framing semantic progress in information-seeking interactions as question-conditioned uncertainty reduction. Our Gaussian embedding-space estimator is reference-free, deterministic, and efficient, while preserving useful structural properties such as compositionality across turns and diminishing returns for redundant evidence. Across MT-Bench, Chatbot Arena, and UltraFeedback, it achieves meaningful agreement with human preferences while reducing the cost of evaluation relative to LLM-as-a-judge baselines. These results suggest that embedding models can support richer dialogue-level evaluation beyond static similarity scoring. When combined with relevance weighting, uncertainty updates, and redundancy discounting, even lightweight embeddings can yield useful dialogue-level evaluation signals. This makes the approach well suited for scalable model comparison, regression testing, and deployment monitoring. At the same time, semantic progress is only one dimension of dialogue quality; the metric does not directly assess factuality, safety, coherence, style, or user satisfaction, and depends on the chosen embedding space and evidence extraction procedure. Future work should combine this signal with complementary evaluators and study its applicability beyond information-seeking dialogue.

\bibliographystyle{unsrtnat}
\bibliography{custom}

\appendix

\clearpage
\appendix
\crefalias{section}{appendix} %
\noindent\rule{\textwidth}{1pt}
\begin{center}
\vspace{7pt}
{\Large  Appendix}
\end{center}
\noindent\rule{\textwidth}{1pt}

\section{Queries as One-dimensional Measurements}
\label{app:queries1d}
We model an embedding direction, such as a query or evidence embedding, as a one-dimensional (rank-one) measurement of the latent semantic state $\mathbf{S} \in \mathbb{R}^d$. A text span $x \in \mathcal A^*$ induces a direction defined by a vector
$\mathbf{z}(x) \in \mathbb{R}^d$ in representation space. The corresponding response is a noisy linear measurement: 
\begin{align}
    y_{t} = \mathbf{z}(x)^\top \mathbf{S} + \varepsilon \quad \varepsilon \sim \mathcal{N}(0, \sigma^2)
    \label{eq:query_as_1d_m}
\end{align}
For intuition, the embedding direction $\mathbf{z}(x)$ specifies which direction of the latent semantic space $\mathbf{S}$ is probed. The scalar $y$ reveals the component of $\mathbf{S}$ along that direction. This measurement admits an equivalent $1$-dimensional operator view:
\begin{align}
    P_x \triangleq \mathbf{z}(x)\mathbf{z}(x)^\top
        \label{eq:query_as_1d_m_2}
\end{align}
so that 
\begin{align}
    P_x\mathbf{S} = \mathbf{z}(x)\bigg(\mathbf{z}(x)^\top\mathbf{S}\bigg)
\end{align}
Thus, although $P_x\mathbf{S}$ is a vector, it is determined solely based on the inner product $\mathbf{z}(x)^\top\mathbf{S}$, which is the quantity in \Cref{eq:query_as_1d_m}. In this sense, each question is a one-dimensional observation ``projection like" measurement of $\mathbf{S}$.

If one wishes $P_x$ to be an idempotent projection operator (i.e., $P^2_x = P_x$) then $\mathbf{z}(x)$ must be normalized and the corresponding projector is:
\begin{align}
    \bm{\Pi}_x \triangleq \frac{ \mathbf{z}(x)\mathbf{z}(x)^\top}{\|\mathbf{z}(x)\|^2}
\end{align}
In our setting however, we intentionally keep $ \mathbf{z}(x)$ unnormalized (to preserve scale relevant for the bayesian update). In that case, we view \Cref{eq:query_as_1d_m_2} as a scaled projector:
\begin{align}
P_x \triangleq \alpha \bm{\Pi}_x \quad \alpha \triangleq \|\mathbf{z}(x)\|^2
\end{align}

\section{Relation to Bayesian Linear Regression}
\label{appendix:blr}

Let \(y_i\in\mathbb{R}\) be the output of a linear function of an input
\(\mathbf{s}_i\in\mathbb{R}^d\):
\begin{align}
    y_i = \mathbf{w}^\top \mathbf{s}_i + \varepsilon_i .
\end{align}
Assume heteroscedastic Gaussian noise
\(\varepsilon_i\sim\mathcal{N}(0,\sigma_i^2)\) and a Gaussian prior
\(\mathbf{w}\sim\mathcal{N}(\mathbf{0},\sigma_p^2\mathbf{I})\).
Let \(\mathbf{S}\in\mathbb{R}^{n\times d}\) be the design matrix whose rows are
\(\mathbf{s}_i^\top\), and let \(\mathbf{y}\in\mathbb{R}^n\) collect the outputs \(y_i\).
Using Bayes' rule and independence of the observations, the log posterior is, up to an additive
constant,
\begin{align}
\log p(\mathbf{w}\mid \mathbf{S},\mathbf{y})
&=
-\frac{1}{2}
\left[
\sigma_p^{-2}\|\mathbf{w}\|_2^2
+
\sum_{i=1}^{n}\sigma_i^{-2}
(y_i-\mathbf{w}^\top\mathbf{s}_i)^2
\right]
+c .
\end{align}
Let
\[
\mathbf{R}=\mathrm{diag}(\sigma_1^2,\ldots,\sigma_n^2).
\]
Then the same expression can be written as
\begin{align}
\log p(\mathbf{w}\mid \mathbf{S},\mathbf{y})
&=
-\frac{1}{2}
\left[
\sigma_p^{-2}\mathbf{w}^\top\mathbf{w}
+
(\mathbf{y}-\mathbf{S}\mathbf{w})^\top
\mathbf{R}^{-1}
(\mathbf{y}-\mathbf{S}\mathbf{w})
\right]
+c \\
&=
-\frac{1}{2}
\left[
\mathbf{w}^\top
(\mathbf{S}^\top\mathbf{R}^{-1}\mathbf{S}
+\sigma_p^{-2}\mathbf{I})
\mathbf{w}
-
2\mathbf{y}^\top\mathbf{R}^{-1}\mathbf{S}\mathbf{w}
\right]
+c .
\end{align}
Therefore the posterior covariance is
\begin{align}
    \mathbf{\Sigma}
    =
    \left(
    \mathbf{S}^\top\mathbf{R}^{-1}\mathbf{S}
    +
    \sigma_p^{-2}\mathbf{I}
    \right)^{-1},
\end{align}
and hence the posterior precision is
\begin{align}
    \mathbf{\Sigma}^{-1}
    =
    \mathbf{S}^\top\mathbf{R}^{-1}\mathbf{S}
    +
    \sigma_p^{-2}\mathbf{I}
    =
    \sigma_p^{-2}\mathbf{I}
    +
    \sum_{i=1}^{n}\sigma_i^{-2}\mathbf{s}_i\mathbf{s}_i^\top .
\end{align}

In our formulation, the input vector \(\mathbf{s}_i\) corresponds to an evidence embedding
\(\mathbf{z}_{t,i}\), and the prior precision \(\sigma_p^{-2}\mathbf{I}\) corresponds to
\(\mathbf{J}_0=\mathbf{\Sigma}_0^{-1}\). Setting
\[
\sigma_i^2 = \frac{\sigma^2}{w_{t,i}}
\quad\text{for } w_{t,i}>0
\]
gives
\[
\mathbf{\Sigma}^{-1}
=
\mathbf{J}_0
+
\sum_i
\frac{w_{t,i}}{\sigma^2}
\mathbf{z}_{t,i}\mathbf{z}_{t,i}^\top ,
\]
which is the weighted precision update in Equation~(6). When \(w_{t,i}=0\), the corresponding
observation has infinite effective noise and contributes no update.

\begin{table}[t]
\centering
\caption{Correspondence between Bayesian linear regression and our weighted Gaussian information-gain formulation.}
\label{tab:blr_correspondence}
\begin{tabular}{ll}
\toprule
\textbf{Bayesian Linear Regression} & \textbf{Our Formulation} \\
\midrule
Latent parameter \(\mathbf{w}\) & Latent semantic state \(\mathbf{S}\) \\
Input vector \(\mathbf{s}_i\) & Evidence embedding \(\mathbf{z}_{t,i}\) \\
Scalar observation \(y_i\) & Unobserved measurement value \(y_{t,i}\) \\
Noise variance \(\sigma_i^2\) & Effective evidence noise \(\sigma^2/w_{t,i}\) \\
Prior covariance \(\sigma_p^2 \mathbf{I}\) & Initial covariance \(\mathbf{\Sigma}_0\) \\
Design matrix \(\mathbf{S}\) & Evidence embedding stack \\
Posterior covariance \(\mathbf{\Sigma}\) & Posterior semantic uncertainty \\
Precision update \(\sigma_i^{-2}\mathbf{s}_i\mathbf{s}_i^\top\) &
Rank-one update \((w_{t,i}/\sigma^2)\mathbf{z}_{t,i}\mathbf{z}_{t,i}^\top\) \\
\bottomrule
\end{tabular}
\end{table}

\section{Proof of \Cref{thm:gaussian-maxent-upper-bound}}
\label{app:proof_max_ent}
\begin{theorem*}
    For a fixed turn $t \in [T]$, let $\mathbf{Z}$ be the random evidence vector corresponding to $\mathbf{z} = \langle \mathbf{z}_1, \ldots,  \mathbf{z}_m \rangle \in \mathbb{R}^{md}$ with evidence components $\mathbf{Z}_i \in \mathbb{R}^d$. If $h(\mathbf{Z} \mid \mathbf{S})\geq c$, where $c$ is independent of the dialogue strategy, and assume the evidence have common empirical covariance $\mathbf{\Sigma}_\mathbf{Z} \succ \mathbf{0}$, then:
    \begin{align}
        I(\mathbf{Z};\mathbf{S}) \leq m \big [\frac{d}{2}\log(2\pi e) + \frac{1}{2}\log\det \mathbf{\Sigma}_\mathbf{Z} \big ] -c 
    \end{align}
\end{theorem*}
\begin{proof}
    Let $t \in [T]$ be a fixed turn and $\bz= \langle \bz_1, \ldots, \bz_m \rangle \in \mathbb{R}^{md}$ be one realization of a random vector $\bZ$. We assume that a given strategy for generating a dialogue defines a conditional distribution $P(\bZ \mid \mathbf{S})$. Together with a distribution $P(\mathbf{S})$ of semantic states, we thus obtain a joint distribution $P(\mathbf{Z}, \mathbf{S})$. The information $\bZ$ reveals on average about $\mathbf{S}$ is given by the mutual information \cite{Cover2006}:
    \begin{align}
        I(\bZ;\mathbf S)=h(\bZ) - h(\bZ \mid \mathbf S),
    \end{align}
    with the differential entropy $h$ defined by $h(X)\triangleq- \int p(x) \log p(x) dx$, where $p(x)$ denotes a density some density in $\mathbb{R}^n$. 
    For (discrete) Shannon entropy we have $\entropy(\bZ\mid \mathbf{S}) \geq 0$ so $ I(\bZ;\mathbf S) \leq \entropy(\bZ)$. In the continuous case, entropy may be negative, so instead assume $h(\bZ \mid \mathbf{S}) \geq c$ with $c$ independent of the dialogue strategy. Then, $ I(\bZ;\mathbf S) \leq h(\mathbf Z)-c$. Now, discretize the evidence-embedding space into a finite alphabet $\mathcal{A}$ labeling 
    cubic bins of size $\delta^d$
    such that
    we can safely approximate  
    $h(\bZ) - h(\bZ \mid \mathbf{S})$ with the corresponding discrete mutual information (which is at least guaranteed if the bins are small relative to the uncertainty of $P(\bZ \mid \mathbf{S}$). The proof uses a source-coding perspective. Since entropy characterizes the minimum achievable average code length, we can upper-bound the entropy of the evidence sequence by constructing an explicit code for it.
    
    Let $\widehat{P}_\mathbf{z}$ be the empirical distribution of the observed sequence $\mathbf{z}$, to reconstruct the vector $\langle \mathbf{z}_1, \ldots, \mathbf{z}_m \rangle$ we can consider a two-part code.

    First, encode the type $\widehat{P}_\mathbf{z}$. Since a type is determined by $|\mathcal A|$ integer counts summing to $m$, the number of possible types is at most $(m+1)^{|\mathcal A|}$ so  encoding the type costs $|\mathcal A|\log (m+1)$ bits.

    Second, given the type $\widehat{P}_\mathbf{z}$, encode which ordered sequence within the type class occurred. The type class has cardinality:
    \begin{align}
|T(\widehat{P}_\mathbf{z})| = \frac{m!}{\prod_{a\in \mathcal A} (m \cdot \widehat{P}_\mathbf{z}(a))!}
    \end{align}
    which is bounded by $|T(\widehat{P}_\mathbf{z})| \leq 2^{m\entropy(\widehat{P}_\mathbf{z})}$ following Theorem 3.1.2 in \cite{Cover2006} with $\epsilon = 0$. Therefore, identifying the ordered realization within the type class requires at most $m\cdot \entropy(\widehat{P}_\mathbf{z})$ bits. Hence, $\bz$ can be described in $m\cdot \entropy(\widehat{P}_\mathbf{z}) + |\mathcal{A}|\log( m+1)$ many bits. Due to Shannon's source coding theorem (Theorem 5.4.1 in \cite{Cover2006}), the entropy cannot be larger than the number of bits needed on the average, we thus have:
    \begin{align}
    \entropy(\bZ)=\entropy(\bZ_1, \ldots, \bZ_m) \leq m \cdot \mathbb{E}[\entropy(\widehat{P}_\bz)] +|\mathcal A|\log (m+1)
    \end{align}
    Next, suppose we retain only the empirical second moment matrix
    \begin{align}
        \widehat{\bm \Sigma}_\bz \triangleq \frac{1}{m} \sum_{i=1}^m \bz_i\bz_i^\top
    \end{align}
We now interpret 
the discrete distribution  as density $\widehat{p}_\bz$ in $\mathbb{R}^d$
that is constant within one bin. Its differential entropy $h(\widehat{p}_\bz)$
cannot be larger 
than the entropy of a Gaussian with zero mean and covariance $\widehat{\bm{\Sigma}}_\bz$, see \cite{Cover2006}, Theorem 8.6.5. Thus,
$h(\widehat{P}_\bz) \leq \frac{d}{2} \log (2\pi e) + \frac{1}{2} \log \det \widehat{\bm{\Sigma}}_\bz$. 
The discrete Shannon entropy $H(\widehat{P}_\bz)$ is thus bounded from above by 
$\frac{d}{2} \log (2\pi e) + \frac{1}{2} \log \det \widehat{\bm{\Sigma}}_\bz + d \log \delta$, where the 
extra summand comes from 
multiplying the density
with $\delta^d$ to get the probability mass function. 

To obtain a lower bound on the discretized version of $P(\bZ \mid \mathbf{S})$ we observe
that averaging the density over bins can only increase entropy and 
thus $H(\bZ \mid \mathbf{S})\geq c + dm \cdot \log \delta$ (with a similar extra term from $dm$-dimensional cubes). 
Obviously, these extra terms cancel for the mutual information. 

Therefore, for a fixed turn, the expression $m\big[ \frac{d}{2}\log(2\pi e) + \frac{1}{2}\log \det \widehat{\bm{\Sigma}}_\bz \big]$ is the empirical Gaussian maximum-entropy surrogate for the entropy of the evidence vector, up to the type-encoding term $|\mathcal A|\log (m+1)$, and hence up to the conditional uncertainty term, and this lower-order coding cost, it gives an upper-bound surrogate for the information the turn's evidence can reveal about the semantic state. 
\end{proof}
\textbf{Remark:} When $|\mathcal A|$ is fixed, the normalized error term $|\mathcal A|\log (m+1)/m$ vanishes as $m\to \infty$. Thus, the Gaussian log-determinant quantity is a large-$m$ upper-bound surrogate.

\section{Proof of \Cref{thm:monotonicity}}
\label{appendix:proof_monotonicity}
\begin{theorem*}[Monotonicity and Telescoping]
Assume $\mathbf{\Sigma}_0 \succ \mathbf{0}$ and weights $\alpha_{t,i}\ge 0$.
Under the precision update in \Cref{eq:precision_update}, we have
$\mathbf{\Sigma}_t \preceq \mathbf{\Sigma}_{t-1}$ (in PSD order) and hence $\IG_t\ge 0$ for all $t$.
Moreover, the total gain telescopes:
\begin{align}
\sum_{t=1}^T \IG_t
=
\frac{1}{2}\log \frac{\det(\mathbf{\Sigma}_0)}{\det(\mathbf{\Sigma}_T)} ,
\end{align}
so it depends only on the initial and final posterior covariances.
\end{theorem*}
\begin{proof}
Assume $\mathbf{\Sigma}_0 \succ \mathbf{0}$ with weights $w_{t, i} \geq 0$, recall 
\begin{align}
    \mathbf{J}_t = \mathbf{J}_{t-1} + \sum_{i=1}^{m_t} \frac{w_{t,i}}{\sigma^2}\,\mathbf{z}_{t,i}\mathbf{z}_{t,i}^\top .
\end{align}
which means $\mathbf{\Sigma}_t\preceq \mathbf{\Sigma}_{t-1}$. Let 
\begin{align}
        \mathbf{A}_t \triangleq\sum_{i=1}^{m_t} \frac{w_{t,i}}{\sigma^2}\,\mathbf{z}_{t,i}\mathbf{z}_{t,i}^\top
\end{align}
Since $w_{t, i} \geq 0$, and $\mathbf{zz^\top}\succeq \mathbf{0}$, each term is PSD, and therefore $\mathbf{A}_t \succeq \mathbf{0}$. The update is 
\begin{align}
    \mathbf{J}_{t} = \mathbf{J}_{t-1} +\mathbf{A}_t
\end{align}
which implies $\mathbf{J}_{t} \succeq \mathbf{J_{t-1}}$. Because $\mathbf{\Sigma}_0 \succ \mathbf{0}$, and $\mathbf{J}_0 \succeq \mathbf{0}$, by induction $\mathbf{J}_t \succ \mathbf{0}$ for all $t$, so $\mathbf{\Sigma}^{-1}_t = \mathbf{J}_t$ is well-defined. Using the fact for any $\mathbf{\Lambda} \succ \mathbf{0}, \mathbf{\Gamma} \succ \mathbf{0}$, and $\mathbf{\Lambda} \succeq \mathbf{\Gamma}$, we have $\mathbf{\Lambda}^{-1} \preceq \mathbf{\Gamma}^{-1}$. Applying to $\mathbf{\Lambda = \mathbf{J}}_t$ and $\mathbf{\Gamma}=\mathbf{J}_{t-1}$, we get $\mathbf{J}_{t}^{-1} \preceq \mathbf{J}_{t-1}^{-1}$ meaning $\mathbf{\Sigma}_{t} \preceq \mathbf{\Sigma}_{t-1}$. Thus, posterior covariance is monotone non-increasing in PSD order.

Next, recall
\begin{align}
    \IG_t 
=
 \frac{1}{2}\log \frac{\det(\mathbf{\Sigma}_{t-1})}{\det(\mathbf{\Sigma}_t)} .
\end{align}

Since, $\mathbf{\Sigma}_{t} \preceq \mathbf{\Sigma}_{t-1}$, determinant monotonicity on the positive-definite cone gives $\det(\bm \Sigma_t)\le \det(\bm \Sigma_{t-1}),$
and hence $\IG_t \geq 0$.
Finally, telescoping follows alternating sums
\begin{align}
    \sum_t \IG_t &= \frac{1}{2}\sum_{t}\log\det(\bm \Sigma_{t-1}) - \log\det(\bm \Sigma_t) \\&= \frac{1}{2}\log\det(\bm \Sigma_{0}) - \log\det(\bm \Sigma_T)\\
    &=\frac{1}{2}\log\frac{\det\bm \Sigma_0}{\det\bm \Sigma_T}
\end{align}
This proves both monotonicity and telescoping.
\end{proof}

\section{Proof of \Cref{thm:submodularity}}
\label{appendix:proof_submodularity}
\begin{theorem*}[Diminishing Returns]
Assume $\mathbf J_0\succ 0$. The set function $F$ is monotone submodular. That is, for all
$A\subseteq B\subseteq\mathcal X $and $x\notin B$,
\begin{align*}
    F(A\cup\{x\})-F(A)
\ge
F(B\cup\{x\})-F(B).
\end{align*}
\end{theorem*}
Let
\begin{align}
    \mathbf{J}(\mathcal{S})
\;=\;
\mathbf{J}_0
\;+\;
\sum_{i \in \mathcal{S}}
\alpha_i \, \mathbf{z}_i \mathbf{z}_i^\top, \quad \mathbf{J}_0 \succ0, \alpha_i \geq 0
\end{align}
and 
\begin{equation}
F(\mathcal{S}) \triangleq \log \det \mathbf{J}(\mathcal{S}).
\end{equation}
We first show $F$ is monotone:
\begin{proof}
First, to show monotonicity: for any $\mathcal{S}$ and element $e \not \in \mathcal{S}$:
\begin{align}
    \mathbf{J}(\mathcal{S}\cup \{e\}) = \mathbf{J}(\mathcal{S})+ \alpha_e \, \mathbf{z}_e \mathbf{z}_e^\top \succeq \mathbf{J}(\mathcal{S})
\end{align}
Since $\log \det (\cdot)$ is increasing over the positive-definite cone under PSD increments, $F(\mathcal{S}\cup \{e\}) \geq F(\mathcal{S})$
\end{proof}
Next, we show $F$ is submodular:
\begin{proof}
    Consider the marginal gain of adding an element $e$:
    \begin{align}
        \Delta_e(\mathcal{S}) &\triangleq F(\mathcal{S}\cup \{e\}) - F(\mathcal{S}) \\ &= \log \frac{\det ( \mathbf{J}(\mathcal{S})+ \alpha_e \, \mathbf{z}_e \mathbf{z}_e^\top)}{\det( \mathbf{J}(\mathcal{S}))}
    \end{align}
    using the matrix determinant lemma 
    \begin{align}
        \det ( \mathbf{J}(\mathcal{S})+ \alpha_e \, \mathbf{z}_e \mathbf{z}_e^\top) = \det( \mathbf{J}(\mathcal S))(1+\alpha_e \, \mathbf{z}_e^\top\mathbf{J}(\mathcal S)^{-1}\mathbf{z}_e)
    \end{align}
    Therefore
    \begin{align}
          \Delta_e(\mathcal{S}) = \log(1+\alpha_e\mathbf{z}_e^\top\mathbf{J}(\mathcal{S})^{-1}\mathbf{z}_e)
    \end{align}
    Now to show submodularity, take $\mathcal{A}\subseteq\mathcal{B}$ which gives us $\mathbf{J}(\mathcal{A})\preceq\mathbf{J(\mathcal{B)}}$ from the earlier proof, implying $\mathbf{J}(\mathcal{A})^{-1}\succeq\mathbf{J}(\mathcal{B})^{-1}$. Then, the following must hold:
    \begin{align}
        \alpha_e\mathbf{z}_e^\top\mathbf{J}(\mathcal{A})^{-1}\mathbf{z}_e \geq \alpha_e\mathbf{z}_e^\top\mathbf{J}(\mathcal{B})^{-1}\mathbf{z}_e
    \end{align}
    because $x \mapsto \log (1+\alpha_ex)$ is increasing for $\alpha_e\geq0$. Therefore, $\Delta_e (\mathcal{A}) \geq \Delta_e(\mathcal{B})$ substituting back we get the conditions for submodularity:
    \begin{align}
        F(\mathcal{A} \cup \{e\}) - F(\mathcal{A})
\;\ge\;
F(\mathcal{B} \cup \{e\}) - F(\mathcal{B}).
\end{align}
\end{proof}

\section{Proofs for Irrelevance and Redundancy Lemmas}
We will be using two important facts:
\begin{itemize}
    \item \textbf{Order Reversal under Inversion}
For any $\mathbf{A} \succ \mathbf{0}, \mathbf{B} \succ \mathbf{0}$, if $\mathbf{A} \prec \mathbf{B}$ then $\mathbf{A}^{-1} \succ \mathbf{B}^{-1}$.
\item \textbf{Rayleigh Quotient Bound}
For a symmetric $\mathbf{M}\succeq \mathbf{0}$ we have $\mathbf{z}^\top\mathbf{M}\mathbf{z} \leq \lambda_{\max}(\mathbf{M})\|\mathbf{z}\|_2^2$.
\end{itemize}

\subsection{Proof of \Cref{lemma:soft_irrelev}}
\label{appendix:soft_irrelev}
\begin{lemma*}[Soft Irrelevance]
Assume embeddings are norm bounded $\|\mathbf{z}\| \leq B$, and $\alpha\leq \varepsilon$.
When $\mathbf{z}$ is irrelevant compared to the question $\mathbf{q}$, then information gain for that turn is bounded by $\log(1+\alpha\lambda_{\max}(\mathbf{J}_0^{-1})B^2) = \mathcal{O}(\varepsilon)$.
\end{lemma*}

\begin{proof}
    Assume $\|\mathbf{z}\| \leq B$, we have
    \begin{align}
        \Delta(\mathbf{z}; \mathbf{J}) =
            \log\!\left(1 + \alpha\, \mathbf{z}^\top \mathbf{J}^{-1}\mathbf{z}\right).
    \end{align}
    Since $\mathbf{J} \succeq \mathbf{J}_0 \succ \mathbf{0}$, by order reversal under inversion, we know  $\mathbf{J}^{-1} \preceq \mathbf{J}^{-1}_0$ and hence
    \begin{align}
\mathbf{z}^\top\mathbf{J}^{-1}\mathbf{z} \leq \mathbf{z}^\top\mathbf{J}^{-1}_0\mathbf{z}
    \end{align}
    Since $\mathbf{J}_0$  is a positive-definite precision matrix, $\mathbf{J_0^{-1}}$ is a symmetric positive-definite covariance matrix. Applying the Rayleigh quotient bound, we get
    \begin{align}
        \mathbf{z}^\top\mathbf{J}^{-1}_0\mathbf{z} \leq \lambda_{\max}(\mathbf{J}_0^{-1})\|\mathbf{z}\|_2^2 \leq \lambda_{\max}(\mathbf{J}_0^{-1})B^2 
    \end{align}
    substituting back into $\Delta(\mathbf{z};\mathbf{J})$ we get 
    \begin{align}
         \Delta(\mathbf{z}; \mathbf{J}) &=
            \log\!\left(1 + \alpha\, \mathbf{z}^\top \mathbf{J}^{-1}\mathbf{z}\right) \\ &\leq \log (1 + \alpha\lambda_{\max}(\mathbf{J}_0^{-1})B^2 )
    \end{align}
    Finally, if $\alpha \leq \varepsilon$, then $ \Delta(\mathbf{z}; \mathbf{J}) \leq \log(1+C\varepsilon) = \mathcal{O}(\varepsilon)$ where $C=\lambda_{\max}(\mathbf{J}_0^{-1})B^2$.
\end{proof}

\subsection{Proof of \Cref{lemma:redundancy}}
\label{appendix:proof_redundancy}
\begin{lemma*}[Redundancy leads to diminishing returns]
Consider repeatedly adding the same $\mathbf{z}$ $k$ times with the same $\alpha$. Then, $\Delta(\mathbf{z};\mathbf{J}_{k}) \geq \Delta(\mathbf{z};\mathbf{J}_{k+1}) $
\end{lemma*}
\begin{proof}
    Let $\mathbf{z} \in \mathbb{R}^d$, $\alpha\geq 0$, define $\mathbf{J}_k=\mathbf{J}_0+k\alpha\mathbf{zz^\top}$ with $\mathbf{J}_0\succ\mathbf{0}$. By definition, we know 
    \begin{align}
        \mathbf{J}_{k+1}=\mathbf{J}_k + \alpha\, \mathbf{z}\mathbf{z}^\top\succeq \mathbf{J}_k \succ \mathbf{0}.
    \end{align}
    By order reversal under inversion we know $\mathbf{J}_{k+1}^{-1} \preceq \mathbf{J}_{k}^{-1}$, and hence
    \begin{align}
\mathbf{z}^\top{\mathbf{J}_{k+1}^{-1}}\mathbf{z}\leq\mathbf{z}^\top{\mathbf{J}_{k}^{-1}}\mathbf{z}
    \end{align}
     Because $x \mapsto \log(1+\alpha x)$ is nondecreasing for $\alpha \geq 0$, we conclude
    \begin{align}
                \Delta(\mathbf{z}; \mathbf{J}_{k+1}) &=
            \log\!\left(1 + \alpha\, \mathbf{z}^\top \mathbf{J}^{-1}_{k+1}\mathbf{z}\right) \\& \leq       
            \log\!\left(1 + \alpha\, \mathbf{z}^\top \mathbf{J}^{-1}_k\mathbf{z}\right) = \Delta(\mathbf{z}; \mathbf{J}_{k}) .
    \end{align}
\end{proof}

\section{Toy Examples}
\label{appendix:toy_examples}

\noindent\textbf{Goal.}
We construct a controlled information-seeking dialogue where \emph{oracle} uncertainty can be computed exactly using a finite hypothesis universe. This allows a direct comparison between (i) oracle information gain computed by eliminating inconsistent hypotheses, and (ii) our Gaussian information gain computed purely in embedding space.

\noindent\textbf{Universe and Prior.}
Let $U$ be a finite set of hypotheses. In our instantiations, $U$ is a fixed list of $N$ cities (city-guessing), a fixed list of $N$ movies (movie-guessing), and more generally QAs related to gardening or workouts, with $N=100$.
We initialize either a uniform prior or a simple non-uniform prior (e.g., population-based city priors):
\begin{align}
p_0(u) &= \frac{1}{N}, \quad u \in U, \label{eq:toy-uniform-prior}\\
\text{H}_0 &= -\sum_{u \in U} p_0(u)\log p_0(u) = \log N. \label{eq:toy-H0}
\end{align}

\noindent\textbf{Dialogue and Evidence.}
We generate a multi-turn dialogue $\{(q_t,a_t)\}_{t=1}^T$ consisting of yes/no questions for city-guessing and movie-guessing, as well as general responses to gardening and workout advice questions.
We compare two variants of equal length:
(i) \textbf{Novel}, where each turn introduces new discriminative information, and
(ii) \textbf{Redundant}, where later turns repeat earlier information (and thus should yield diminishing or zero marginal gain under the oracle).

\noindent\textbf{Oracle Consistency Update.}
Let $e_t$ denote the evidence revealed at turn $t$ (e.g., the semantic constraint implied by $(q_t,a_t)$).
We maintain a feasible set $S_t \subseteq U$ of hypotheses consistent with all evidence so far:
\begin{align}
S_0 &= U, \label{eq:toy-S0}\\
S_t &= \Bigl\{u \in S_{t-1} \;\big|\; u \text{ is consistent with } e_t \Bigr\}. \label{eq:toy-St}
\end{align}
In our implementation, since we know the universe, we can manually construct the consistency predicates at each turn given the question and answers, mapping dialogue history to a subset of surviving hypotheses, and we enforce $S_t \subseteq S_{t-1}$ when the turns are not irrelevant or redundant. 

\noindent\textbf{Oracle Entropy and Information Gain.}
Under the uniform posterior over $S_t$, the oracle entropy is
\begin{align}
\text{H}_t &= \log |S_t|. \label{eq:toy-Ht}
\end{align}
The per-turn oracle information gain is then
\begin{align}
\mathrm{IG}^{\text{oracle}}_t
&= \text{H}_{t-1} - \text{H}_t
= \log \frac{|S_{t-1}|}{|S_t|} \;\ge\; 0, \label{eq:toy-ig}
\end{align}
and the cumulative oracle gain telescopes:
\begin{align}
\sum_{t=1}^T \mathrm{IG}^{\text{oracle}}_t
&= \log \frac{|S_0|}{|S_T|}
= \log \frac{N}{|S_T|}. \label{eq:toy-telescope}
\end{align}
For redundant turns that add no new constraint, we have $S_t = S_{t-1}$ and thus $\mathrm{IG}^{\text{oracle}}_t = 0$.

\noindent\textbf{Comparison to \GaussianIG.}
For the same dialogues, we compute our Gaussian information gain using only question--answer embeddings and closed-form covariance updates (\Cref{eq:precision_update}).
We plot per-turn gain and cumulative gain for both the oracle and Gaussian metrics; across controlled settings, dialogues that introduce novel constraints for more turns achieve larger cumulative gain, while redundant variants exhibit diminished marginal gain. Refer to \Cref{fig:redundancy_recovery}, \Cref{fig:city}, \Cref{fig:movies}, \Cref{fig:gardening} for results. Note that the user and LLM here are synthetic.
\begin{figure*}[t]
  \centering
  \small
\includegraphics[width=0.8\textwidth]{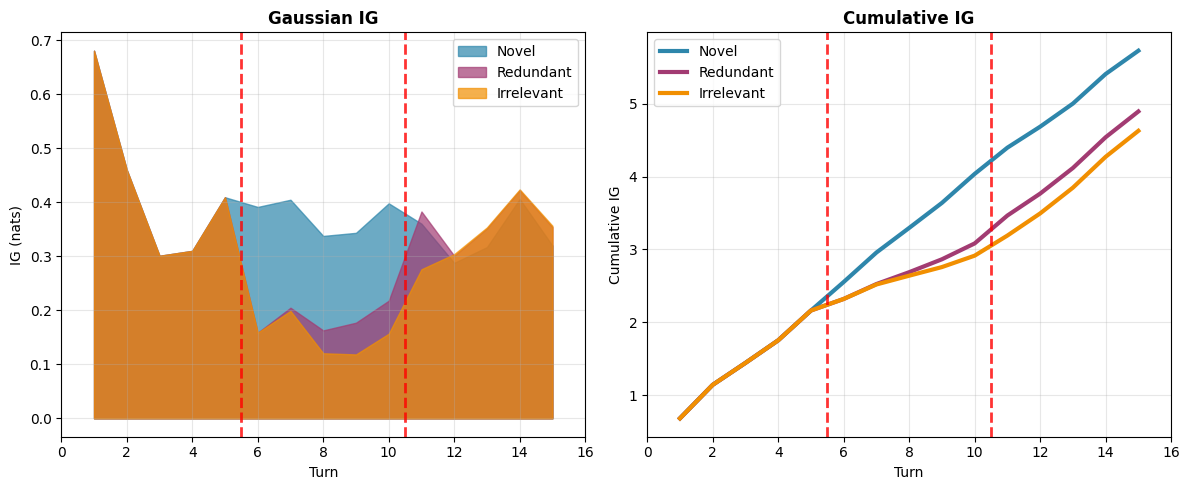}
  \caption{Information gain distinguishes dialogue quality in a 15-turn synthetic setting.
  We compare dialogues that introduce novel, redundant, or irrelevant information.
  (Left) Per-turn Gaussian information gain: redundant and irrelevant turns yield lower marginal gain than novel turns.
  (Right) Cumulative information gain: dialogues with a higher proportion of novel information achieve consistently higher total gain despite approximation noise.}
  \label{fig:redundancy_recovery}
\end{figure*}
We can see per-turn information gain decreases during the redundant and irrelevant phase, reflecting diminishing returns when new evidence does not substantially reduce uncertainty.
When novel information is reintroduced, the marginal gain increases again.
While the Gaussian approximation does not exactly match an oracle uncertainty model, it consistently preserves the dialogue-level ordering: conversations that introduce novel information for a larger fraction of turns achieve higher cumulative information gain.
This demonstrates that the metric captures meaningful differences in long-horizon dialogue progress without relying on learned judges or supervision. 

\begin{figure*}[t]
  \centering
  \includegraphics[width=\textwidth]{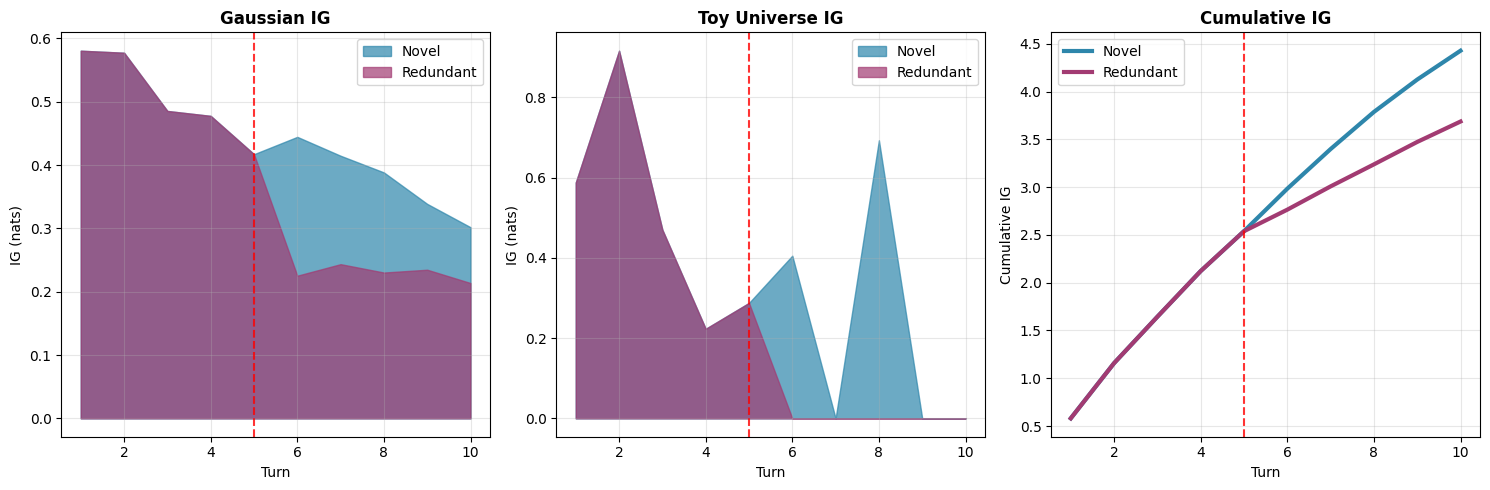}
  \caption{Toy example with the city guessing universe. The LLM tries to guess a city by asking questions and the user provides Yes/No to the questions. We compare the case where at turn 5, the LLM starts asking redundant questions to if it kept asking novel questions, in an ideal universe we know $S_t=S_{t-1}$ so information gain is $0$, in our Gaussian Approximation, although information gain does not drop to $0$, we see a clear difference between if the questions asked was novel or redundant.}
  \label{fig:city}
\end{figure*}

\begin{figure*}[t]
  \centering
  \includegraphics[width=\textwidth]{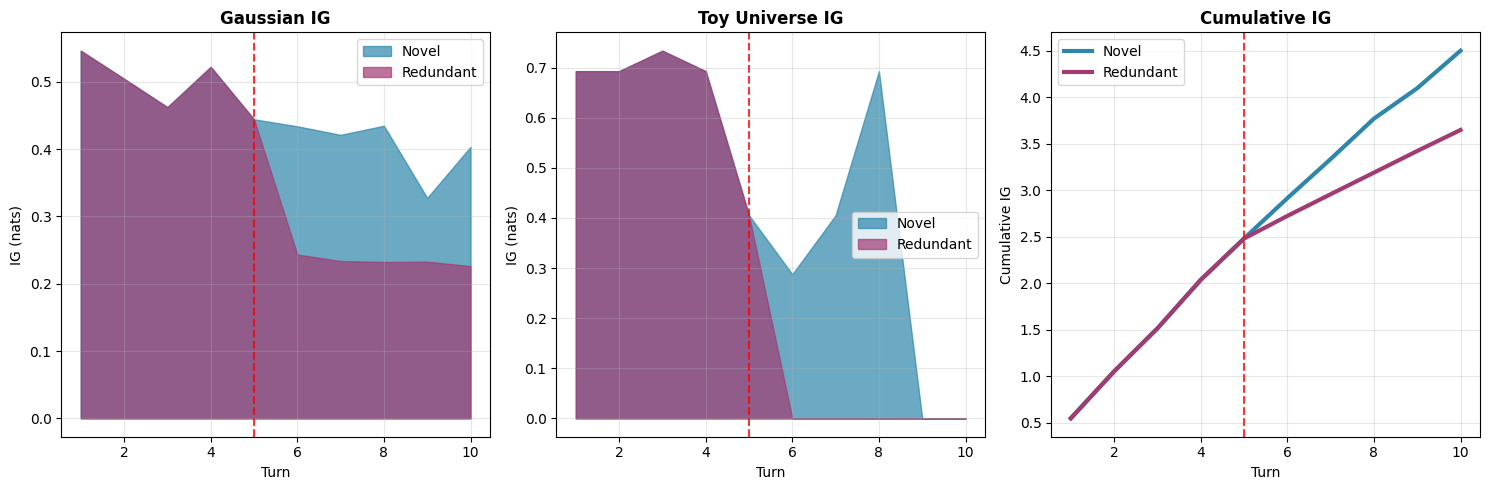}
  \caption{Toy example with the movie guessing universe. The LLM tries to guess a movie by asking questions and the user provides Yes/No to the questions. We compare the case where at turn 5, the LLM starts asking redundant questions to if it kept asking novel questions. The trends observed are similar to \Cref{fig:city}.}
  \label{fig:movies}
\end{figure*}

\begin{figure*}[t]
  \centering
  \includegraphics[width=\textwidth]{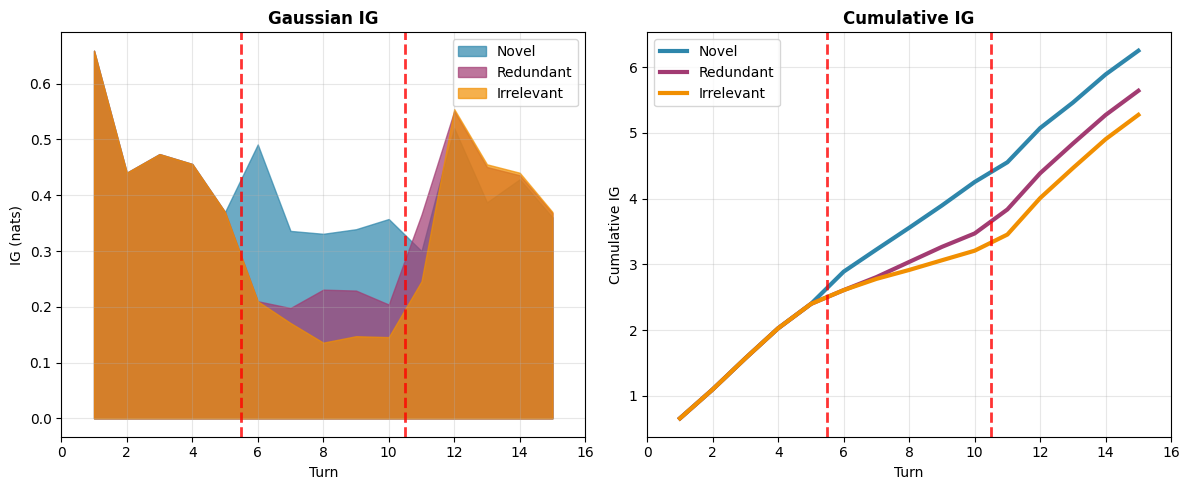}
  \caption{Toy example where the user asks open questions for advice related to gardening. The LLM tries to answer the question. We compare the cases where between turns 5 and 10, where LLM answers are either novel, redundant, or irrelevant. We observe similar trends as \Cref{fig:city}, \Cref{fig:movies}, and also that information gain recovers after the LLM starts asking novel questions again.}
  \label{fig:gardening}
\end{figure*}

\section{Experiment Details} \label{app:exp}

\paragraph{Algorithm and Parameters.} \Cref{alg:ig} presents the estimator using an explicit evidence-level loop for readability. In implementation, the per-turn rank-one updates can be vectorized as $\mathbf{J}'=\mathbf{J}+\sigma^{-2}\mathbf{Z}_t^\top \operatorname{diag}(\mathbf{w}_t)\mathbf{Z}_t$, where rows of $\mathbf{Z}_t$ are evidence embeddings. In \Cref{tab:experimental_settings}, we report the parameters we used.

\paragraph{Existing assets and licenses.}
We use existing public benchmarks and models only for research evaluation. MT-Bench and Chatbot Arena are used following their public release terms~\cite{zheng2023judging}; UltraFeedback is used following its public release terms~\cite{10.5555/3692070.3692454}. Embedding models and LLM baselines are credited through their original papers or model cards. We do not redistribute any existing datasets, model checkpoints, or generated annotations.

\paragraph{LLM Scoring.} We use pointwise scalar judging without explanations for all LLM baselines to standardize API outputs and runtime measurement. Pairwise prompts may improve these baselines, especially for reasoning-specialized models. Thus, these baselines should be interpreted as operational judge configurations rather than the best possible performance of each underlying LLM. Since the evaluation set excludes human-tie cases, predicted ties are counted as incorrect: when the two candidate dialogues receive the same 1--10 score, the judge fails to identify the preferred dialogue.

\begin{algorithm}[t]
\caption{\GaussianIG: Gaussian Information Gain Evaluation }
\label{alg:ig}
\begin{algorithmic}[1]
\Input Dialogue turns $(q_t,a_t)_{t=1}^T$, encoder $\mathrm{enc}$, parameters $\sigma_0,\sigma,\beta,\eta$
\Output Dialogue score $\widehat{\mathcal P}_q(D_{1:T})$

\State $\mathbf{J} \gets \sigma_0^{-2}\mathbf{I}$, \quad $\widehat{\mathcal P} \gets 0$ \Comment{Initialize prior precision}
\For{$t=1,\ldots,T$}
    \State $\{Z_{t,i}\}_{i=1}^{m_t} \gets \mathrm{Split}(a_t)$ \Comment{Extract answer evidence}
    \State $\mathbf{q}_t \gets \mathrm{enc}(q_t)$, \quad $\widehat{\mathbf q}_t \gets \mathbf{q}_t/\|\mathbf{q}_t\|_2$ \Comment{Embed current user need}
    \For{$i=1,\ldots,m_t$}
        \State $\mathbf{z}_{t,i} \gets \mathrm{enc}(Z_{t,i})$, \quad
        $\widehat{\mathbf z}_{t,i} \gets \mathbf{z}_{t,i}/\|\mathbf{z}_{t,i}\|_2$ \Comment{Embed evidence unit}
        \State $w_{t,i} \gets \max(0,\langle \widehat{\mathbf q}_t,\widehat{\mathbf z}_{t,i}\rangle)^\beta$ \Comment{Question relevance}
        \State \textbf{if} $w_{t,i} < \eta$ \textbf{ then } $w_{t,i} \gets 0$ \Comment{Discard weak evidence}
    \EndFor
    \State $\mathbf{J}' \gets \mathbf{J} + \sum_{i=1}^{m_t} \frac{w_{t,i}}{\sigma^2}\mathbf{z}_{t,i}\mathbf{z}_{t,i}^{\top}$ \Comment{Precision update}
    \State $\widehat{\mathrm{IG}}_t \gets \frac{1}{2}\left(\log\det \mathbf{J}' - \log\det \mathbf{J}\right)$ \Comment{Turn-level information gain}
    \State $\mathbf{J} \gets \mathbf{J}'$, \quad $\widehat{\mathcal P} \gets \widehat{\mathcal P} + \widehat{\mathrm{IG}}_t$ \Comment{Accumulate progress}
\EndFor
\State \Return $\widehat{\mathcal P}$
\end{algorithmic}
\end{algorithm}

\begin{table}[t]
\centering
\small
\caption{Default experimental settings for Gaussian information-gain evaluation.}
\label{tab:experimental_settings}
\begin{tabular}{ll}
\toprule
\textbf{Setting} & \textbf{Value} \\
\midrule
Default encoder & \texttt{Qwen/Qwen3-Embedding-0.6B} via \texttt{SentenceTransformer} \\
Hardware & Apple M1 Pro, 10-core CPU, 16-core GPU, 16GB RAM \\
Device backend & MPS \\
Batch size & 32 \\
Encoder mode & Evaluation mode \\
Evidence segmentation & Deterministic sentence/bullet chunking \\
Minimum chunk length & 12 characters \\

Embedding normalization & L2-normalized before relevance scoring \\
Relevance weight & $r^{\mathrm{rel}}_{t,i}=\max(0,\langle \widehat{\mathbf q}_t,\widehat{\mathbf z}_{t,i}\rangle)^\beta$ \\
Relevance exponent & $\beta=1.0$ \\
Relevance cutoff & $\eta=0.05$ \\
Quality factor & $r^{\mathrm{qual}}_{t,i}=1$ \\
Projection dimension & $k=d$; no random projection \\
Prior variance & $\sigma_0^2=1.0$ \\
Observation variance & $\sigma^2=0.25$ \\
Log-determinant computation & \texttt{numpy.linalg.slogdet} \\
Numerical jitter & $10^{-12}$, only if log-det sign is nonpositive \\
Runtime measurement & End-to-end wall-clock time including embedding and Gaussian updates \\
\bottomrule
\end{tabular}
\end{table}

\noindent\textbf{Timing Protocol.}
Let $\mathcal{D}=\{(D^A_i,D^B_i)\}_{i=1}^N$ denote a fixed ordered subset of
MT-Bench dialogue pairs.
For each method, we measure runtime using Python wall-clock timers
(\texttt{time.perf\_counter}) as follows:
\begin{enumerate}
    \item A fixed random seed is used to select and order evaluation examples.
    \item Timing starts immediately before the first scoring call.
    \item Each dialogue pair is scored sequentially.
    \item Timing stops after the final prediction is produced.
\end{enumerate}
\begin{table}[t]
\centering
\caption{
Runtime variability across $R=5$ runs.
Values report mean $\pm$ standard deviation of end-to-end wall-clock time (seconds) for evaluating $N=100$ dialogue pairs.
}
\label{tab:runtime_std}
\begin{tabular}{lcc}
\toprule
\textbf{Method}
& \textbf{MT-Bench}
& \textbf{Chatbot Arena} \\
\midrule
Mistral Large 3
& $131.8 \pm 11.8$
& $118.9 \pm 2.0$ \\

DeepSeek R1
& $633.1 \pm 9.2$
& $479.6 \pm 3.2$ \\

GPT OSS 120b
& $568.0 \pm 198.0$
& $328.3 \pm 94.7$ \\

Claude Sonnet 3.7
& $259.9 \pm 13.1$
& $250.9 \pm 23.9$ \\

Claude Sonnet 4
& $219.6 \pm 12.7$
& $175.4 \pm 8.4$ \\

Claude Sonnet 4.5
& $341.9 \pm 16.8$
& $286.3 \pm 12.8$ \\

\midrule
\textbf{Ours (\GaussianIG)}
& \textbf{$86.4 \pm 6.8$}
& \textbf{$44.1 \pm 5.3$} \\
\bottomrule
\end{tabular}
\end{table}
We report total elapsed time and normalized runtime (seconds per dialogue
pair), averaged over $R=5$ runs.
We further verify that overhead from data access and prediction logic is
negligible relative to scoring time. The standard deviations for the runtimes are reported in \Cref{tab:runtime_std}.

\noindent\textbf{Judge Prompt:}

\begin{quote}
\small
You are an expert evaluator of LLM responses.

Evaluate the quality of the following response with respect to the
dimension: {\texttt{\{DIMENSION\_DEFINITION\}}}.

{Conversation:}

\{\texttt{CONVERSATION}\}

Rate the response on a scale from 1 to 10,

Do not reward answers simply for being longer. Focus on quality relative to the specified dimension.

Return ONLY a single integer from 1 to 10. Do not include any explanation or additional text.
\end{quote}

The dimension definitions injected into the prompt are:

\begin{itemize}

\item \textbf{Helpfulness:} The response should directly and usefully address the user's request, provide relevant information, and avoid unnecessary or unhelpful content.

\item \textbf{Instruction following:} The response should follow the user's explicit and implicit instructions, including requested format, scope, constraints, and task requirements.

\item \textbf{Truthfulness:} The response should avoid factual errors, unsupported claims, and misleading statements.

\item \textbf{Honesty:} The response should accurately represent uncertainty, limitations, and whether the requested information is known or inferable.
\end{itemize}

\section{Why Larger Embedding Models Might Not Help}
\label{app:large_models}
From \Cref{eq:delta_ig}, the per-turn gain depends only on the scalar
$\mathbf{z}^\top \mathbf{J}_{t-1}^{-1}\mathbf{z}$, measuring alignment with directions of remaining uncertainty.
Diagonalizing $\boldsymbol{\Sigma}_{t-1}=\mathbf{J}_{t-1}^{-1}=\mathbf{U}\boldsymbol{\Lambda}\mathbf{U}^\top$ and writing $\mathbf{z}=\mathbf{U}\mathbf{c}$ yields
\begin{align}
    \IG(\mathbf{z})=\frac{1}{2}\log (1+\alpha\sum_i \lambda_i c_i^2),
\end{align}
where $\lambda_i$ are the eigenvalues of $\boldsymbol{\Sigma}_{t-1}$.
Crucially, information gain depends on the \emph{effective rank} of remaining uncertainty rather than the raw dimensionality or parameter count of the embedding model.
Once embeddings capture the dominant semantic axes relevant to the task, larger models primarily refine representations within directions where $\lambda_i$ is already small, yielding negligible additional volume reduction.
As a result, increased model capacity does not necessarily translate into higher information gain.

\section{Synthetic Experiment Details} \label{app:syn}
\subsection{Padding}
\label{app:padding}

\begin{table}[h]
\centering
\caption{
Length and padding robustness on UltraFeedback on embedding model \cite{minishlab2024model2vec}. Each entry reports the mean absolute score change $\Delta_{c,r}$, with the relative change shown in parentheses.
}
\label{tab:padding_potion}
\begin{tabular}{rccc}
\toprule
$r$ & Irrelevant & Redundant & Novel relevant \\
\midrule
0.00 & $0.000$ $(0.00\%)$ & $0.000$ $(0.00\%)$ & $0.000$ $(0.00\%)$ \\
0.25 & $0.000$ $(0.00\%)$ & $0.049$ $(0.76\%)$ & $1.076$ $(25.27\%)$ \\
0.50 & $0.000$ $(0.00\%)$ & $0.060$ $(1.04\%)$ & $1.860$ $(39.64\%)$ \\
1.00 & $0.001$ $(0.02\%)$ & $0.066$ $(1.53\%)$ & $3.234$ $(64.85\%)$ \\
2.00 & $0.012$ $(0.19\%)$ & $0.068$ $(1.77\%)$ & $5.184$ $(103.78\%)$ \\
\bottomrule
\end{tabular}
\end{table}

\begin{table}[h]
\centering
\caption{
Length and padding robustness on UltraFeedback on embedding model \cite{wang2020minilmdeepselfattentiondistillation}. Each entry reports the mean absolute score change $\Delta_{c,r}$, with the relative change shown in parentheses.
}
\label{tab:padding_mini}
\begin{tabular}{rccc}
\toprule
$r$ & Irrelevant & Redundant & Novel relevant \\
\midrule
0.00 & $0.000$ $(0.00\%)$ & $0.000$ $(0.00\%)$ & $0.000$ $(0.00\%)$ \\
0.25 & $0.000$ $(0.00\%)$ & $0.083$ $(1.19\%)$ & $1.130$ $(23.62\%)$ \\
0.50 & $0.004$ $(0.05\%)$ & $0.100$ $(1.54\%)$ & $1.978$ $(38.13\%)$ \\
1.00 & $0.010$ $(0.15\%)$ & $0.106$ $(2.10\%)$ & $3.583$ $(66.50\%)$ \\
2.00 & $0.032$ $(0.48\%)$ & $0.111$ $(2.45\%)$ & $5.697$ $(105.07\%)$ \\
\bottomrule
\end{tabular}

\end{table}

\begin{table}[h]
\centering
\caption{Length and padding robustness on UltraFeedback on embedding model \cite{qwen3embedding}. Each entry reports the mean absolute score change $\Delta_{c,r}$, with the relative change shown in parentheses.}
\label{tab:padding_qwen0.6}
\begin{tabular}{rccc}
\toprule

$r$ & Irrelevant & Redundant & Novel relevant \\

\midrule

0.00 & $0.000$ $(0.00\%)$ & $0.000$ $(0.00\%)$ & $0.000$ $(0.00\%)$ \\

0.25 & $0.037$ $(0.82\%)$ & $0.050$ $(0.47\%)$ & $1.183$ $(23.69\%)$ \\

0.50 & $0.037$ $(0.82\%)$ & $0.067$ $(0.73\%)$ & $2.127$ $(39.58\%)$ \\

1.00 & $0.036$ $(0.81\%)$ & $0.055$ $(0.51\%)$ & $3.714$ $(68.25\%)$ \\

2.00 & $0.062$ $(1.19\%)$ & $0.069$ $(0.75\%)$ & $6.589$ $(122.46\%)$ \\

\bottomrule

\end{tabular}

\end{table}

A potential concern is that $\widehat{\mathcal P}$ may reward response length rather than semantic progress. To isolate the effect of added content, we construct controlled perturbations of real UltraFeedback prompt-response pairs. For each base pair $(q,a)$, we form a padded response
\[
a^{(c,r)} = a \oplus p^{(c,r)},
\]
where $\oplus$ denotes string concatenation, $c$ denotes the padding condition, and $r$ is the target padding ratio. The padding text is chosen so that
\[
|p^{(c,r)}| \approx r|a|,
\]
where $|\cdot|$ denotes token length. We then measure the score change
\[
\Delta_{c,r}
=
\widehat{\mathcal P}_q(a^{(c,r)})
-
\widehat{\mathcal P}_q(a).
\]

We consider three padding conditions. \emph{Irrelevant padding} is sampled from responses to unrelated prompts and filtered to have low embedding similarity to $q$. \emph{Redundant padding} repeats evidence units already present in the original response $a$. \emph{Novel relevant padding} is sampled from other responses to the same prompt and filtered to be relevant to $q$. The filtering is used only to construct controlled perturbation sets; the scoring procedure itself is unchanged.

This setup separates three possible effects of added text: generic length increase, repetition of existing evidence, and addition of new question-relevant evidence. If $\widehat{\mathcal P}_q$ were primarily a length proxy, all three padding conditions would increase the score similarly as $r$ grows. Instead, we find that irrelevant and redundant padding produce only small score changes, while novel relevant padding produces a substantially larger increase. These results indicate that $\widehat{\mathcal P}_q$ is sensitive to whether added content contributes new question-relevant evidence, rather than simply rewarding longer responses. Tables \ref{tab:padding_potion}, \ref{tab:padding_mini}, \ref{tab:padding_qwen0.6}, show the results with $r\in \{0.00, 0.25, 0.50, 1.00, 2.00\}$. We evaluate it on $n=100$ examples, yielding a total of $1500$ dialogues.

\subsection{Hallucination}
\label{app:synth}

\begin{figure*}[t]
\centering
\begin{subfigure}{0.32\textwidth}
\includegraphics[width=\linewidth]{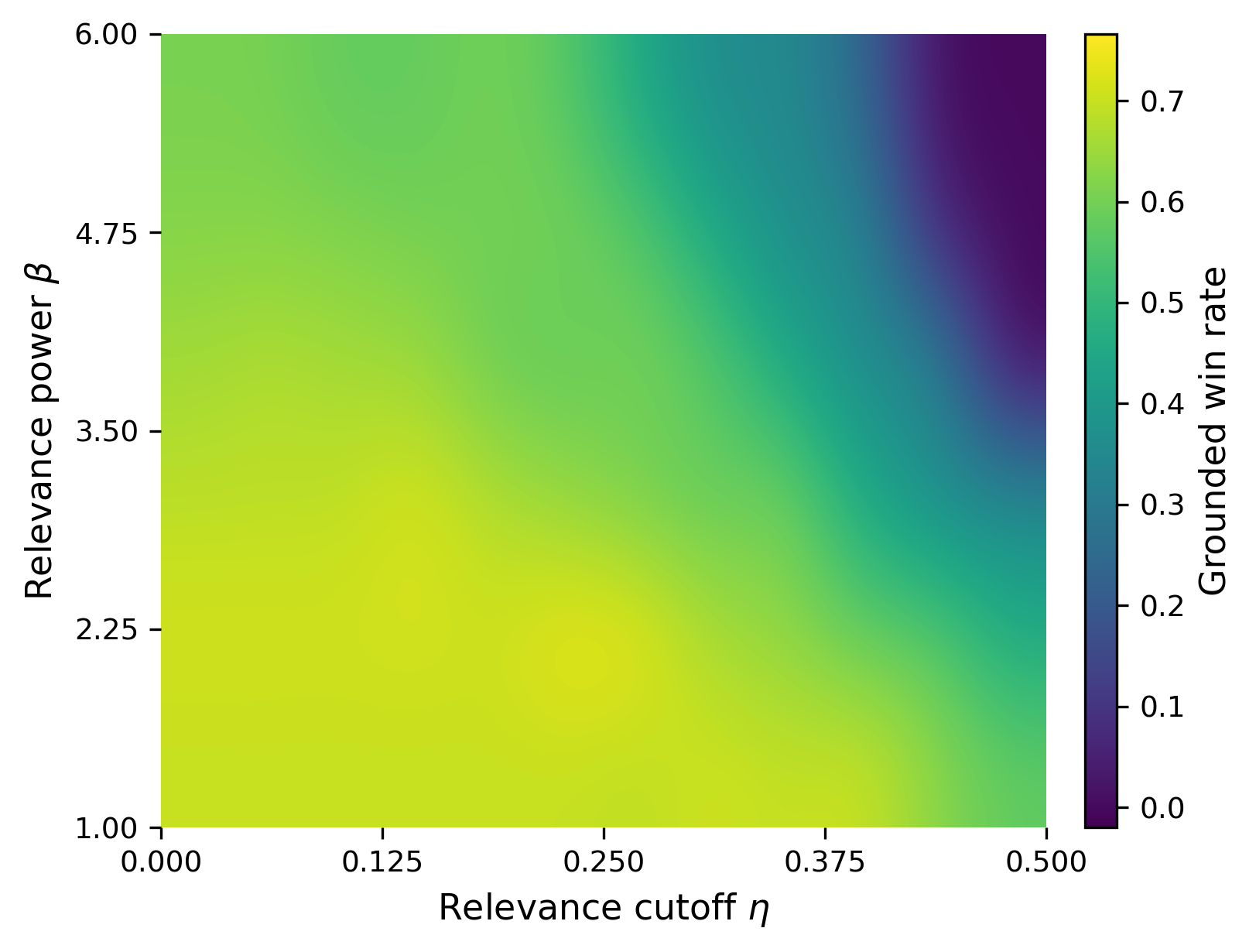}
\caption{potion-base-2M}
\end{subfigure}
\begin{subfigure}{0.32\textwidth}
\includegraphics[width=\linewidth]{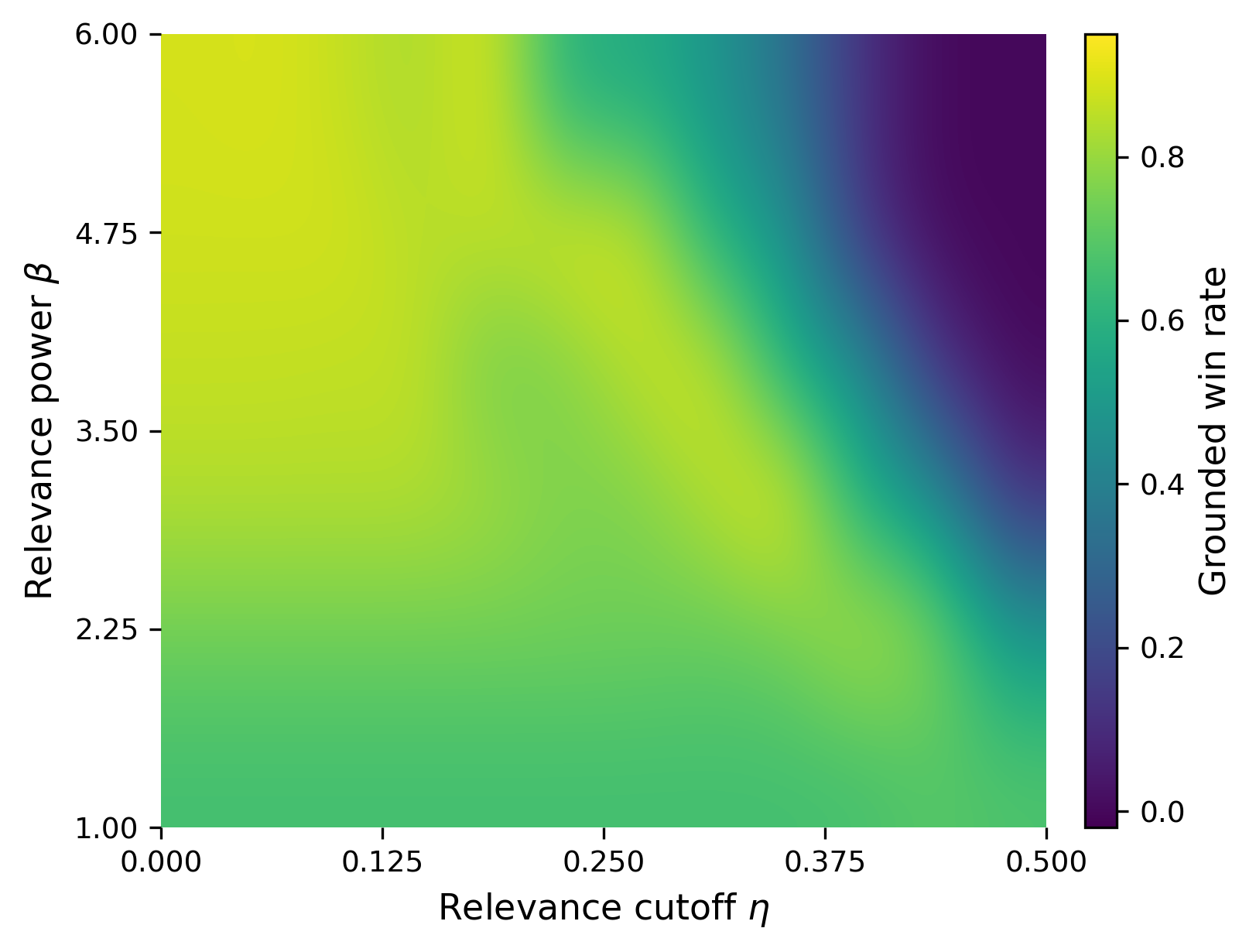}
\caption{Qwen3-Embedding-0.6B}
\end{subfigure}
\begin{subfigure}{0.32\textwidth}
\includegraphics[width=\linewidth]{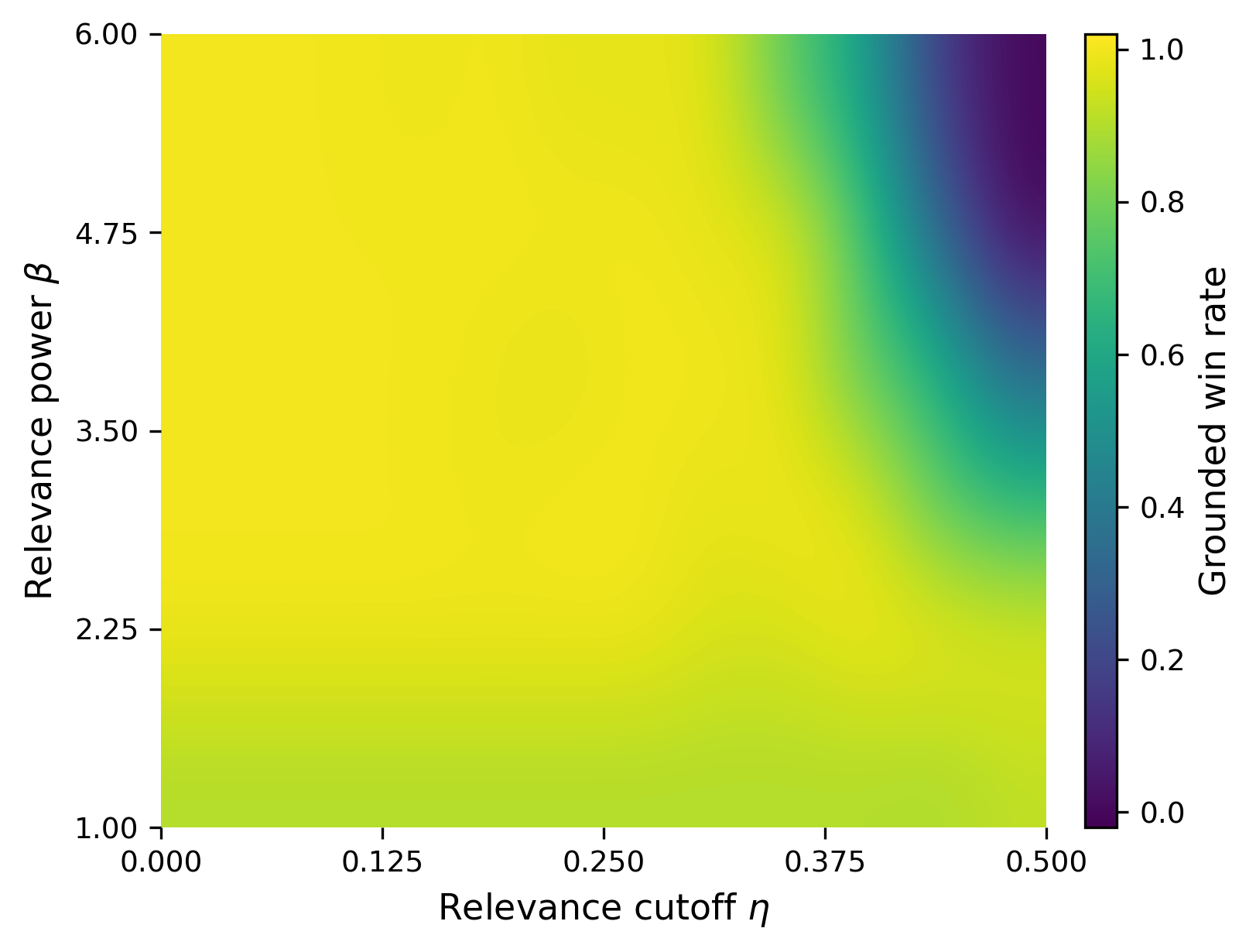}
\caption{Qwen3-Embedding-4B}
\end{subfigure}

\caption{
Sensitivity of grounded win rate to the relevance cutoff $\eta$ and weighting exponent $\beta$ in the synthetic hallucination stress test.
}
\label{fig:eta_beta}
\end{figure*}

Let each domain define a set of valid entity pairs $(x, y) \in \mathcal{F}_d$ where $x$ is the query entity and $y$ is the correct answer entity. For example in the country-capital domain, $(\mathrm{Germany}, \mathrm{Berlin}) \in \mathcal{F}_d$. 

For each pair $(x, y)$, we instantiate a factual response from a fixed template $T_d(x, y)$. We then construct a hallucinated counterpart by replacing $y$ with an entity $y' \in \mathcal{Y}_d$ where $y' \neq y$ and $(x, y') \not \in \mathcal{F}_d$. This would give us a matched pair $$T_d(x, y) \quad \text{ vs. } \quad T_d(x, y')$$ where both responses share the same query entity and surface structure, but only the first contains the correct factual relation.

To increase linguistic diversity, we then use LLMs OpenAI's GPT 5.4 \cite{openai2026gpt54} only as paraphrasers, the entities $x, y, y'$ are fixed, and the model is instructed not to alter them. Thus the correctness is determined by the manually specified factual relation set $\mathcal{F}_d$, not by the LLM.

For each matched pair, we compute \GaussianIG for the factual and corrupted responses under a sweep of relevance cutoffs $\eta$ and relevance exponents $\beta$. We report grounded win rate, defined as the fraction of pairs for which the factual response receives a higher semantic-progress score than the corrupted response:
\begin{align}
    \mathrm{WinRate} = \frac{1}{N}\sum_{i=1}^N \mathbb{I}\left[\widehat{P}_q(D^{\mathrm{fact}}_i) > \widehat{P}_q(D^{\mathrm{corr}}_i) \right].
\end{align}
Ties are counted as non-wins. \Cref{fig:eta_beta} shows that sensitivity to factual perturbations
depends on both the embedding model and the relevance parameters. For
Qwen3-Embedding-0.6B, $\beta=1$ gives a stable win rate around $0.66$, while
moderate sharpening with larger $\beta$ raises performance above $0.85$ and
peaks at $0.93$. However, large cutoffs can sharply reduce performance,
indicating that overly aggressive filtering removes useful evidence.

Qwen3-Embedding-4B is more robust, reaching near-perfect win rates across many
settings, while Potion-2M is weaker and degrades under high cutoffs or large
exponents. Overall, the stress test shows that \GaussianIG can prefer
grounded evidence over semantically similar corrupted evidence, but that this
behavior depends on embedding quality and parameter calibration. We therefore
treat factuality evaluation as complementary to semantic-progress evaluation.

\section{Frequently Asked Questions}
\noindent\textbf{What does reference-free mean in this work?}
Reference-free means that the metric does not require human-written ground-truth answers, gold references, or human annotations at evaluation time. The score is computed solely from the dialogue history using fixed embeddings and closed-form updates.

\noindent\textbf{Why assume a Gaussian model in embedding space?}
The Gaussian assumption is a modeling choice that enables a simple, closed-form approximation to information gain with strong theoretical guarantees. Under a Gaussian posterior, information gain reduces to a log-determinant of the covariance matrix, yielding monotonicity, telescoping across turns, and submodularity under rank-one updates.

We do not claim that semantic uncertainty is truly Gaussian; rather, the Gaussian serves as a tractable surrogate that captures relative uncertainty volume in embedding space.

Empirically, we find that this approximation is sufficient to produce a stable and informative evaluative signal across embedding models of widely varying capacity.

\noindent\textbf{Why does uncertainty always decrease under your Bayesian update?}
Our formulation fixes the observation noise variance and updates only the posterior uncertainty over a fixed latent semantic state.
Each update adds a positive semi-definite rank-one matrix to the precision matrix, which guarantees that posterior covariance is non-increasing in Loewner order. This differs from hierarchical Bayesian models in which noise variance or model parameters are inferred and posterior uncertainty may increase.

\noindent\textbf{Can Bayesian updates ever increase uncertainty in general?}
Yes.
In Bayesian models that infer observation noise or higher-level hyperparameters, posterior uncertainty can increase when observations are inconsistent with prior assumptions. Our model intentionally excludes this behavior by fixing the noise variance, enabling a fast, monotone, and tractable approximation to information gain.

\noindent\textbf{Does higher information gain imply factual correctness?}
No. The metric measures epistemic progress rather than correctness. Confident but incorrect responses may still reduce uncertainty in embedding space. We view this limitation as shared with many automatic evaluators and consider factuality checks a complementary signal.

\end{document}